\documentclass[11pt]{article}

\usepackage[T1]{fontenc}
\usepackage{lmodern}

\usepackage[margin=1in]{geometry}
\usepackage[english]{babel}
\usepackage{appendix}

\usepackage{amsmath,amssymb,amsthm,mathtools,mathrsfs,bbm}
\usepackage{relsize}

\usepackage{graphicx}
\usepackage{subcaption}
\usepackage{booktabs}
\usepackage{placeins}
\usepackage{tikz}
\usetikzlibrary{plotmarks}
\usepackage{pgfplots}
\usepgfplotslibrary{groupplots}
\pgfplotsset{compat=1.18}

\usepackage{enumitem}
\usepackage{xcolor}
\usepackage[]{natbib}
\usepackage{fancyhdr}
\usepackage{titlesec}
\usepackage{pifont}
\usepackage{hyperref}

\hypersetup{
    colorlinks=true,
    linkcolor=black,
    citecolor=black,
    urlcolor=blue
}
\allowdisplaybreaks

\newcommand{\E}{\mathbb{E}}
\newcommand{\eps}{\varepsilon}
\renewcommand{\epsilon}{\eps}

\newcommand{\cmark}{\ding{51}}
\newcommand{\xmark}{\ding{55}}

\DeclareMathOperator*{\argmin}{arg\,min}

\newtheorem{assum}{Assumption}
\newtheorem{defi}{Definition}
\newtheorem{thm}{Theorem}
\newtheorem{prop}{Proposition}
\newtheorem{lem}{Lemma}
\newtheorem{cor}{Corollary}

\newtheorem{rem}{Remark}

\titleformat{\section}{\large\bfseries}{\thesection}{1em}{}
\title{Finite-Time Analysis of the Natural Policy Gradient in
Finite-Horizon Markov Decision Processes}

\author{%
Asha Barua\textsuperscript{1}, 
Sajad Khodadadian\textsuperscript{1*}
}

\date{}

\begin{document}

\maketitle

\begingroup
\renewcommand{\thefootnote}{\fnsymbol{footnote}}
\footnotetext[1]{Corresponding author.} 
\endgroup
\footnotetext[1]{Grado Department of Industrial and Systems Engineering,
Virginia Polytechnic Institute and State University,
Blacksburg, Virginia 24061, USA.
E-mail: \texttt{\{ashabarua,sajadk\}@vt.edu}.}

\setcounter{footnote}{1}

\begin{abstract}
Natural Policy Gradient (NPG) is a well-established Reinforcement Learning algorithm that underlies widely used methods such as Trust Region Policy Optimization and Proximal Policy Optimization, both of which have demonstrated strong empirical success. In this paper, we study exact NPG in finite-horizon Markov Decision Processes with known dynamics and horizon-dependent transition kernels. We provide the first finite-time convergence guarantees for this algorithm in this setting, for which we consider both constant and increasing step size regimes. With a constant step size $\eta_t=\eta$, we prove that NPG converges sublinearly with a rate of $\mathcal{O}(H^{2}/t)$ after $t$ iterations, where $H$ is the horizon length. We also extend this constant step size analysis to linear MDPs in an exact population-projection oracle under a full support projection distribution, recovering the same sublinear rate as in the tabular
setting. Furthermore, with increasing step sizes, we prove that this algorithm achieves a linear convergence rate of $\mathcal{O}\left(\left(1-\frac{1}{\vartheta_\rho}\right)^t\right)$ for a problem-dependent constant $\vartheta_\rho > 1$, and the horizon-only robust schedule of the form $\eta_t=\eta_0(H/(H-1))^t$ where $\eta_0>0$ and $H \geq 2$, attains this same geometric rate.
\end{abstract}

\section{Introduction}
Reinforcement Learning (RL)
has achieved substantial empirical success in domains such as robotic control,
game playing, and autonomous navigation
\citep{singh2022reinforcement,silver2016mastering,arulkumaran2017deep}.
Among RL methods, policy gradient algorithms are widely used for
high-dimensional control. The Natural Policy Gradient (NPG)
\citep{kakade2001natural}, in particular, accounts for the geometry of the
policy space and underlies widely used practical methods such as Trust Region Policy
Optimization (TRPO) \citep{schulman2015trust} and Proximal Policy Optimization
(PPO) \citep{schulman2017proximal}. Mathematically, MDPs characterize the underlying structure of RL algorithms.

MDPs are commonly formulated over either an infinite horizon, with average or discounted return objectives, or a finite horizon, with decisions made over a fixed number of time steps \citep{puterman2014markov,jin2020provably}. Existing theoretical analyses of policy gradient methods have primarily
considered infinite-horizon discounted MDPs, for which both asymptotic
convergence \citep{kakade2001natural} and finite-time guarantees
\citep{fazel2018global,agarwal2021theory} are well developed. Finite-horizon
MDPs \citep{puterman2014markov,jin2020provably}, however, arise naturally in
sequential planning, robotic manipulation, and RL for large language models
\citep{Bai2022TrainingAH,ouyang2022training,
zhan2023policy, SHAKYA2023120495, bhandari2024global, zhang2025survey}. Despite their practical relevance,
finite-time convergence guarantees for exact NPG in finite-horizon MDPs remain
unavailable, even in the tabular setting.

The finite-horizon setting introduces two structural difficulties. First, an
optimal policy is generally nonstationary and therefore consists of a sequence
of horizon-dependent decision rules. Consequently, updating the horizon-$h$
decision rule can alter the state distributions at subsequent horizons
$h+1,\ldots,H$, coupling policy improvement across the horizon. Second, the analyses in the
discounted setting commonly exploit the strict contraction induced by the
discount factor. No analogous discount-induced contraction is available in
the finite-horizon setting, where policies, transition kernels, and value
functions may depend on the horizon index.

In this paper, we study the finite time convergence of NPG for finite horizon MDPs. To the best of our knowledge, this is the first finite-time analysis of
exact NPG in finite-horizon MDPs with known, horizon-dependent transition
kernels and nonstationary policies; the closest finite-horizon convergence analysis concerns vanilla softmax policy gradient
\citep{klein2024beyond}, whose rates are sublinear, derived from smoothness and
weak Polyak-{\L}ojasiewicz inequalities with model-dependent
constants. Our main
contributions are as follows:
\begin{itemize}
    \item \textbf{Constant step size:}
    For tabular MDPs, we prove that NPG with a constant step size attains an
    optimality gap of order $\mathcal{O}(H^2/t)$ after $t$ iterations. Under
    the linear MDP assumption of \citep{jin2020provably}, the same rate holds in
    an oracle regime with exact action-value evaluation and exact
    population projection under a full support projection distribution.

    \item \textbf{Increasing step sizes:}
    For tabular MDPs, we prove that increasing step sizes yield an optimality
    gap of order
    $\mathcal{O}((1-1/\vartheta_\rho)^t)$, where $\vartheta_\rho>1$ is a
    problem-dependent distribution-mismatch coefficient. We further show that the horizon-only schedule
$\eta_t=\eta_0(H/(H-1))^t$, with $\eta_0>0$ and $H \geq 2$, satisfies the step size growth
condition required for geometric convergence; in the best case
$\vartheta_\rho=H$, it yields the rate
$\mathcal{O}\!\left(\left(1-\frac{1}{H}\right)^t\right)$.

    \item \textbf{Numerical illustration:} Finally, we present 
    simulations that illustrate the convergence behavior predicted by the
    theoretical bounds.
\end{itemize}
\subsection{Related work}\label{sec:Related_work}

 In this section, we provide a brief overview of the existing literature most
relevant to our analysis. Table~\ref{tab:comparison} compares our results with prior work.

\begin{table}[htbp]
    \caption{Finite-time convergence guarantees for NPG and PMD. The NPG update considered here is equivalent to KL-based PMD \citep{shani2020adaptive,xiao2022convergence}. ``Horizon-dep.\ $P^h$''
    indicates whether the analysis permits horizon-dependent transition
    kernels. Sublinear and linear denote, respectively, an
    $\mathcal{O}(1/t)$ optimality gap and a geometric
    $\mathcal{O}(c^t)$ optimality gap for some $c<1$. Step size regimes and
    rates are paired within each row.}
    \resizebox{\textwidth}{!}{%
        \begin{tabular}{llccc}
            \toprule
            \textbf{Paper} & \textbf{Algorithm} & \textbf{Step size}
            & \textbf{Horizon-dep.\ $P^h$} & \textbf{Rate} \\
            \midrule
            \citet{agarwal2021theory}
                & NPG & Constant & \xmark & Sublinear \\
            \citet{khodadadian2022linear}
                & NPG & Constant \& Adaptive & \xmark
                & Sublinear \& Linear \\
            \citet{xiao2022convergence}
                & PMD & Constant \& Increasing & \xmark
                & Sublinear \& Linear \\
            \citet{lan2023policy}
                & PMD & Constant \& Increasing & \xmark
                & Sublinear \& Linear \\
            \citet{liu2024elementary}
                & NPG & Constant & \xmark & Linear \\
            \midrule
            \textbf{This paper}
                & \textbf{NPG (KL-PMD)}
                & \textbf{Constant \& Increasing}
                & \cmark
                & \textbf{Sublinear \& Linear} \\
            \bottomrule
        \end{tabular}%
    }
    \label{tab:comparison}
\end{table}

NPG was introduced by \citet{kakade2001natural}, and its finite-time behavior,
together with that of related policy gradient methods, is now well understood
for tabular infinite-horizon discounted MDPs
\citep{agarwal2021theory,bhandari2024global}. Sublinear
$\mathcal{O}(1/t)$ guarantees have been established under constant step sizes,
whereas increasing or adaptive step sizes can yield geometric convergence
\citep{khodadadian2022linear,xiao2022convergence,
shani2020adaptive,lan2023policy}. Geometric convergence has also been obtained
under constant step sizes \citep{liu2024elementary} and through regularization
\citep{mei2020global,cen2022fast,zhan2023policy}. These analyses concern
infinite-horizon discounted MDPs with stationary transition kernels and
therefore do not cover finite-horizon MDPs with horizon-dependent dynamics. We establish the corresponding sublinear and
geometric guarantees for exact NPG in this setting, providing a baseline for
regularized and sample-based finite-horizon extensions.

\section{Preliminaries}\label{sec:prelim}

\subsection{Finite-Horizon Markov Decision Process}
We consider a finite-horizon Markov Decision Process (MDP) represented by the tuple
$(\mathcal{S},\mathcal{A},H,\mathcal{P},\mathcal{R})$, where $\mathcal{S}$ and $\mathcal{A}$
are finite state and action sets, respectively, and $H$ is the horizon length. The transition kernels are
$\mathcal{P}=\{P^h:\mathcal{S} \times \mathcal{A} \to \Delta({\mathcal{S}}) \mid h=1,2,\dots,H \}$, where $P^h(s'\mid s,a)$ denotes the probability of transitioning to $s'$ from
state-action pair $(s,a)$ at horizon $h$, and $\Delta(\mathcal{S})$ is the
probability simplex over $\mathcal{S}$. The reward
functions are $\mathcal{R}=\{R^h: \mathcal{S} \times \mathcal{A} \to [0,1] \mid h=1,2,\dots,H\}$. 

In the finite-horizon setting, an optimal policy is generally nonstationary. We therefore
represent a policy as $\pi=(\pi^1,\ldots,\pi^H)$, where
$\pi^h(\cdot\mid s)\in\Delta(\mathcal{A})$ is the action distribution at
horizon $h$ in state $s$. For $h\in[H]$ and $s\in\mathcal{S}$, we define
\begin{align*}
V^{\pi,h}(s)
:=
\mathbb{E}\!\left[
\sum_{i=h}^{H}R^i(S^i,A^i)
\mid
S^h=s,\;
A^i\sim\pi^i(\cdot\mid S^i),\;
S^{i+1}\sim P^i(\cdot\mid S^i,A^i),\;
i=h,\ldots,H
\right].
\end{align*}
We adopt the convention $V^{\pi,H+1}(s)=0$ for every policy $\pi$ and
$s\in\mathcal{S}$. Since $R^h(s,a)\in[0,1]$,  we have
$0\le V^{\pi,h}(s)\le H-h+1$ for all $(\pi,h,s)$. The corresponding action-value function is
\begin{align*}
Q^{\pi,h}(s,a)
&:=
R^h(s,a)
+
\sum_{s'\in\mathcal{S}}
P^h(s'\mid s,a)V^{\pi,h+1}(s'),
\end{align*}
and the advantage function is defined as $A^{\pi,h}(s, a) = Q^{\pi,h} (s,a) -V^{\pi,h} (s)$.
The goal of the agent is to find an optimal policy $\pi^\star$ such that for every horizon $h$, every state $s$ and any policy $\pi$,
\begin{equation} \nonumber
V^{\pi^\star,h}(s) \geq  V^{\pi,h} (s).
\end{equation}
Throughout the paper, a superscript $\star$ denotes a quantity associated with the optimal policy $\pi^\star$.

We analyze the Natural Policy Gradient (NPG), an iterative algorithm that finds the optimal policy through a smooth form of policy iteration \citep{kakade2001natural}.  In iteration $t$, given a constant step size $\eta>0$, the update of the NPG policy takes the form: 
\begin{equation}\label{eq:NPG}
\pi_{t+1}^h(a\mid s)
=
\frac{
\pi_t^h(a\mid s)
\exp\!\left(\eta Q^{\pi_t,h}(s,a)\right)
}{
Z_t^h(s)
},
\qquad
h\in[H],\ s\in\mathcal{S},\ a\in\mathcal{A},
\end{equation}
where the normalization constant is $Z_t^h(s) = {\sum_{a'\in\mathcal{A}} \pi^h_t (a' \mid s) \exp( \eta  Q^{\pi_t , h} (s,a'))}$. The NPG update \eqref{eq:NPG} is equivalently expressed as the KL-based policy
mirror descent (PMD) update
\citep{shani2020adaptive,lan2023policy}
\begin{equation}\label{eq:PMD}
\pi_{t+1}^h(\cdot\mid s)
=
\arg\max_{p\in\Delta(\mathcal{A})}
\left\{
\eta
\left\langle Q^{\pi_t,h}(s,\cdot),p\right\rangle
-
D_{\mathrm{KL}}\!\left(
p\mid\mid\pi_t^h(\cdot\mid s)
\right)
\right\}.
\end{equation}
Both forms of the update are applied independently to every horizon-state pair $(h,s)$.
We use the multiplicative form \eqref{eq:NPG} in the sublinear analysis and
the variational form \eqref{eq:PMD}, with an iteration-dependent step size
$\eta_t$, in the geometric analysis.

Furthermore, for any horizon $h \in [H]$ and every $i \geq h$, we define the state visitation distribution induced by a policy $\pi$ as
\begin{equation} \label{eq:SVD}
d^{h\to i,\pi}_s(\tilde s)
:= \mathbb{P}\Big(
S^i=\tilde s \mid
S^h=s,\;
A^j\sim \pi^j(\cdot\mid S^j),\;
S^{j+1}\sim P^j(\cdot\mid S^j,A^j),\;
j=h,\ldots,i-1
\Big).
\end{equation}
\subsection{Linear MDP}
Beyond the tabular setting, we consider the linear MDP model of
\citet{jin2020provably}, which represents transition kernels and rewards using
a low-dimensional feature map.
\begin{assum}\label{ass:LinMDP}
An MDP is linear with respect to a feature map
$\phi:\mathcal{S}\times\mathcal{A}\to\mathbb{R}^d$ if, for every $h\in[H]$, there exist a vector-valued function
$\omega^h=(\omega_1^h,\ldots,\omega_d^h):\mathcal{S}\to\mathbb{R}^d$
and a vector $\zeta^h\in\mathbb{R}^d$ such that, for every
$(s,a)\in\mathcal{S}\times\mathcal{A}$,
\[
P^h(\cdot\mid s,a)
=
\left\langle\phi(s,a),\omega^h(\cdot)\right\rangle,
\qquad
R^h(s,a)
=
\left\langle\phi(s,a),\zeta^h\right\rangle,
\]
where
$\langle\phi(s,a),\omega^h(\cdot)\rangle:=\sum_{j=1}^d\phi_j(s,a)\,\omega_j^h(\cdot)$.
We further assume that $\|\phi(s,a)\|\le B$ for some $B>0$ and every
$(s,a)\in\mathcal{S}\times\mathcal{A}$, and that
$\max\{\|\sum_{s\in\mathcal{S}}\omega^h(s)\|,\|\zeta^h\|\}\le\sqrt{d}$ for
every $h\in[H]$\footnote{Throughout the paper, $\|\cdot\|$ denotes the
Euclidean norm.}.
\end{assum}
Under Assumption~\ref{ass:LinMDP}, for every policy $\pi$, there exist vectors
$\{w^{\pi,h}\}_{h=1}^H\subset\mathbb{R}^d$ such that
\begin{equation}\label{eq:Q_w}
Q^{\pi,h}(s,a)
=
\left\langle\phi(s,a),w^{\pi,h}\right\rangle
\qquad
\text{for every }(s,a)\in\mathcal{S}\times\mathcal{A}.
\end{equation}
The linear
representation is most useful when $d\ll|\mathcal{S}||\mathcal{A}|$; choosing
$\phi(s,a)$ as the standard basis of $\mathbb{R}^{|\mathcal{S}||\mathcal{A}|}$ recovers the tabular model. We parameterize the horizon-dependent policy by
$\theta=(\theta^1,\ldots,\theta^H)$, where $\theta^h\in\mathbb{R}^d$, as
\begin{equation}\label{eq:policy_parametrization}
\pi_\theta^h(a\mid s)
=
\frac{
\exp\!\left(\left\langle\phi(s,a),\theta^h\right\rangle\right)
}{
\sum_{a'\in\mathcal{A}}
\exp\!\left(\left\langle\phi(s,a'),\theta^h\right\rangle\right)
}.
\end{equation}
Writing $\pi_t:=\pi_{\theta_t}$, in parameter space, the Q-NPG update
\citep{agarwal2021theory} with constant step size $\eta>0$ is
\begin{equation}\label{eq:Q-NPG-update-oracle}
\theta_{t+1}^h
=
\theta_t^h+\eta\, w^{\pi_t,h},
\qquad h\in[H].
\end{equation}
Section~\ref{sec:linear} specifies an exact population-projection oracle under
which \eqref{eq:Q-NPG-update-oracle} induces precisely the tabular NPG update
\eqref{eq:NPG} in policy space.

The remainder of the paper is organized as follows : Section~\ref{sec:sublinear_convergence} establishes sublinear convergence of
NPG with a constant step size in the tabular (Section~\ref{sec:tabular}) and
linear (Section~\ref{sec:linear}) MDP settings.
Section~\ref{sec:linear_convergence} establishes geometric convergence under
increasing step sizes. Section~\ref{sec:Simulation} presents the simulation
results, and Section~\ref{sec:conclusion} concludes. The main text contains
the key lemmas, while detailed proofs and auxiliary results are deferred to
the Appendix.

\section{Sublinear Convergence}\label{sec:sublinear_convergence}
\subsection{Tabular MDP}\label{sec:tabular}
In this section, we first establish that NPG with a constant step size $\eta>0$ in
tabular MDPs achieves a finite-time global convergence rate of
$\mathcal{O}((H-h+1)^2/t)$ at any given horizon $h$ after $t$ iterations. The
following lemmas provide the main ingredients for proving this result.

\begin{lem}(Performance Difference Lemma)\label{lem:PDL}
 Consider a finite-horizon MDP and any pair of policies $\pi$ and $\pi'$. 
For every horizon $h \in [H]$ and state $s \in \mathcal{S}$, we have
\[
V^{\pi,h}(s)-V^{\pi',h}(s)
=
\sum_{i=h}^H
\sum_{\tilde{s}\in\mathcal{S},\,\tilde{a}\in\mathcal{A}}
d_s^{h\to i,\pi}(\tilde{s})\,
\pi^i(\tilde{a}\mid\tilde{s})\,
A^{\pi',i}(\tilde{s},\tilde{a}).
\]
\end{lem}

The next lemma lower bounds the one-step improvement of policies through the NPG update rule.

\begin{lem}(Improvement Lower Bound for NPG) \label{lem:ILB}
Consider the NPG update \eqref{eq:NPG} with step size $\eta>0$, and $t\ge0$. Suppose that $\pi_t$ has full support, i.e.,
$\pi_t^i(a\mid s)>0$ for every $i\in[H]$, $s\in\mathcal{S}$, and
$a\in\mathcal{A}$. Then, for every $h\in[H]$ and $s\in\mathcal{S}$,
\[
V^{\pi_{t+1},h}(s)-V^{\pi_t,h}(s)
\ge
\frac{1}{\eta}\log Z_t^h(s)-V^{\pi_t,h}(s)
\ge0.
\]
\end{lem}

Thus, the value functions are nondecreasing along the policy sequence. In
particular, $V^{\pi^\star,h}(s)-V^{\pi_t,h}(s)\ge
V^{\pi^\star,h}(s)-V^{\pi_T,h}(s)$ for every $t\le T$; see
Lemma~\ref{lemma:Vk} in the Appendix. We now state the first main result.

\begin{thm}(Global Convergence of NPG)\label{thm:GC}
Consider a finite-horizon MDP and the NPG update \eqref{eq:NPG} with constant
step size $\eta>0$ and uniform initialization
$\pi_0^h(\cdot\mid s)=\mathrm{Unif}(\mathcal{A})$\; for every $h\in[H]$ and
$s\in\mathcal{S}$. Then, for every integer $T\ge1$, horizon $h\in[H]$, and
state $s\in\mathcal{S}$,
\[
V^{\pi^\star,h}(s)-V^{\pi_T,h}(s)
\le
\frac{(H-h+1)\log|\mathcal{A}|}{\eta T}
+
\frac{(H-h+1)^2}{T}.
\]
Consequently, for any fixed $h\in[H]$, if
$\eta\ge\frac{\log|\mathcal{A}|}{H-h+1}$, then
$V^{\pi^\star,h}(s)-V^{\pi_T,h}(s)\le\frac{2(H-h+1)^2}{T}$ for every
$s\in\mathcal{S}$, and an optimality gap of at most $\varepsilon$ is attained
at horizon $h$ after $T\ge 2(H-h+1)^2/\varepsilon$ iterations.
\end{thm}
The $H^2$ factor in Theorem~\ref{thm:GC} arises from
Lemma~\ref{lem:nk}, where the value potential is bounded by
$\sum_{i=H-j}^H(H-i+1)=\frac{(j+1)(j+2)}{2}\le(j+1)^2$. We do not know
whether this quadratic dependence is unavoidable for constant step size NPG
or is an artifact of our proof technique. In infinite-horizon discounted
MDPs, \citet{johnson2023optimal} proved matching upper and lower bounds for
exact PMD; a matching lower bound for the finite-horizon setting remains an
open problem.

Because the tabular model is known, its action-value functions can be
evaluated exactly. As $\eta\to\infty$, the NPG update approaches the greedy policy-improvement
step of exact policy iteration with respect to $Q^{\pi_t,h}$, whereas a finite
$\eta$ yields a smooth policy-improvement step. We treat the tabular setting as an exact
known-model baseline and next study the structured linear MDP setting as an
exact-oracle benchmark for future sample-based extensions.

\subsection{Linear MDP} \label{sec:linear}

We now identify an exact oracle regime in which the constant step size
guarantee of Theorem~\ref{thm:GC} carries over to the linear MDP setting.
Throughout this subsection, we work under Assumption~\ref{ass:LinMDP} and use
the linear representation \eqref{eq:Q_w}, the softmax parametrization
\eqref{eq:policy_parametrization}, and the Q-NPG update
\eqref{eq:Q-NPG-update-oracle}. Note that the softmax parametrization ensures
that $\pi_t$ has full support over actions, i.e., $\pi_t^i(a\mid s)>0$ for
every $i\in[H]$, $s\in\mathcal S$, and $a\in\mathcal A$, at every iteration
$t\ge0$.

To specify the oracle, fix a projection distribution $v$ over
$\mathcal S\times\mathcal A$ with full support:
\begin{equation}
\label{eq:v_full_support_linear}
    v(s,a)>0,
    \qquad
    \forall (s,a)\in\mathcal S\times\mathcal A .
\end{equation}
Here $v$ is a design choice rather than an environmental quantity; in the
finite state-action setting, \eqref{eq:v_full_support_linear} can always be
satisfied by taking, for example, $v=\mathrm{Unif}(\mathcal S\times\mathcal A)$. We adopt this projection formulation because its empirical counterpart arises
naturally in sample-based extensions, where coverage and estimation error must
be addressed. For fixed policy parameters $\theta$ and
horizon $h\in[H]$, define the population projection loss
\begin{equation}
\label{eq:linear_population_projection_loss}
    L_h(w;\theta,v)
    :=
    \mathbb E_{(s,a)\sim v}
    \left[
    \big(
    Q^{\pi_\theta,h}(s,a)-\langle w,\phi(s,a)\rangle
    \big)^2
    \right].
\end{equation}
By the linear realizability identity \eqref{eq:Q_w}, the minimum of
$L_h(\cdot;\theta,v)$ is attained and equals zero. Moreover, since $v$ has
full support, every exact minimizer represents the action-value function
pointwise: if $w\in\argmin_{u\in\mathbb R^d}L_h(u;\theta,v)$, then
\[
\sum_{s\in\mathcal S}\sum_{a\in\mathcal A}
v(s,a)
\big(
Q^{\pi_\theta,h}(s,a)-\langle w,\phi(s,a)\rangle
\big)^2
=0,
\]
and since each weight $v(s,a)$ is strictly positive and every summand is
nonnegative, every summand must vanish. At iteration $t\ge0$, the exact
population-projection oracle returns
\begin{equation}
\label{eq:linear_projection_oracle_choice}
    w^{\pi_t,h}
    \in
    \argmin_{w\in\mathbb R^d}L_h(w;\theta_t,v),
    \qquad h\in[H],
\end{equation}
so that
\begin{equation}
\label{eq:linear_oracle_exact_at_t}
    \langle w^{\pi_t,h},\phi(s,a)\rangle
    =
    Q^{\pi_t,h}(s,a),
    \qquad
    \forall (s,a)\in\mathcal S\times\mathcal A .
\end{equation}
The full support condition on $v$ is used only to turn zero projection error
into the pointwise identity \eqref{eq:linear_oracle_exact_at_t}, and
non-uniqueness of the minimizer is harmless: all exact minimizers induce the
same values $\langle w^{\pi_t,h},\phi(s,a)\rangle$ on every state-action pair,
and the policy update depends only on these values. full support is also not
the only sufficient condition for \eqref{eq:linear_oracle_exact_at_t}: if
$\Sigma_v:=\mathbb E_{(s,a)\sim v}[\phi(s,a)\phi(s,a)^\top]$ is positive
definite, then zero loss gives
$(w-w^{\pi_\theta,h})^\top\Sigma_v(w-w^{\pi_\theta,h})=0$, hence
$w=w^{\pi_\theta,h}$. We use the full support condition because it provides a direct finite-state
argument without requiring additional assumptions on the feature map.

Consequently, the Q-NPG parameter update \eqref{eq:Q-NPG-update-oracle}
induces the policy update
\begin{align}
\pi_{t+1}^h(a\mid s)
&=
\frac{
\pi_t^h(a\mid s)
\exp\!\left(\eta\langle w^{\pi_t,h},\phi(s,a)\rangle\right)
}{
\sum_{a'}\pi_t^h(a'\mid s)
\exp\!\left(\eta\langle w^{\pi_t,h},\phi(s,a')\rangle\right)
}
=
\frac{
\pi_t^h(a\mid s)
\exp\!\left(\eta Q^{\pi_t,h}(s,a)\right)
}{
\sum_{a'}\pi_t^h(a'\mid s)
\exp\!\left(\eta Q^{\pi_t,h}(s,a')\right)
},
\label{eq:linear_equals_tabular_update}
\end{align}
where the second equality uses \eqref{eq:linear_oracle_exact_at_t}. That is,
under the oracle \eqref{eq:linear_projection_oracle_choice}, Q-NPG induces
exactly the tabular NPG update \eqref{eq:NPG} in policy space, and the
finite-time guarantee follows by applying Theorem~\ref{thm:GC} to the induced
policy sequence.

\begin{prop} \label{prop:linMDP}
Consider the Q-NPG update \eqref{eq:Q-NPG-update-oracle} with constant step
size $\eta>0$ under Assumption~\ref{ass:LinMDP}. Suppose that $v$ satisfies
the full support condition \eqref{eq:v_full_support_linear} and that, at every
iteration $t\ge0$ and horizon $h\in[H]$, the oracle returns an exact minimizer
as in \eqref{eq:linear_projection_oracle_choice}. Let $\theta_0^h=\mathbf 0$
for every $h\in[H]$, so that $\pi_0^h(\cdot\mid s)=\mathrm{Unif}(\mathcal A)$
for every $h\in[H]$ and $s\in\mathcal S$. Then, for every integer $T\ge1$,
horizon $h\in[H]$, and state $s\in\mathcal S$,
\[
V^{\pi^\star,h}(s)-V^{\pi_T,h}(s)
\le
\frac{(H-h+1)\log|\mathcal A|}{\eta T}
+
\frac{(H-h+1)^2}{T}.
\]
\end{prop}

\begin{rem}\label{rem:linMDP}
Proposition~\ref{prop:linMDP} identifies an exact oracle regime in which
linear Q-NPG induces the same policy-space dynamics and achieves the same
convergence order as tabular NPG. Within this regime, moving from the
tabular model to the linear setting changes only the representation of the
update direction. In sample-based extensions, $w^{\pi_t,h}$ is estimated from finite data, so
the pointwise identity \eqref{eq:linear_oracle_exact_at_t} may not hold
exactly. The analysis must therefore account for estimation error and the
coverage of $v$.

Across the $H$ horizons, storing the update directions
$\{w^{\pi_t,h}\}_{h=1}^H$ uses $\mathcal{O}(dH)$ memory, compared with
$\mathcal{O}(|\mathcal S||\mathcal A|H)$ for a full tabular action-value array.
This comparison concerns only the representation of the update directions.
Implementing exact action-value evaluation and the oracle
\eqref{eq:linear_projection_oracle_choice} may still require access to the
complete feature map, model dynamics, or state-action-level quantities.
Therefore, Proposition~\ref{prop:linMDP} does not establish an end-to-end
$\mathcal{O}(dH)$ memory or computational guarantee. Similarly, the absence of
$d$ from the bound follows from the exact-oracle assumption and does not imply
a dimension-independent implementation guarantee. Estimating $w^{\pi_t,h}$
from data would introduce dependence on the feature dimension and data
coverage, together with statistical error terms.
\end{rem}

\section{Geometric Convergence}\label{sec:linear_convergence}
In this section, we apply the PMD update \eqref{eq:PMD} with an
iteration-dependent step size $\eta_t>0$ in place of the constant step size
$\eta$, and fix an initial state distribution $\rho\in\Delta(\mathcal S)$.
For any horizon $h\in[H]$, define
\[
V^{\pi,h}(\rho) \;:=\; \mathbb E_{S^h\sim \rho}\!\big[V^{\pi,h}(S^h)\big],
\]
the expected return from horizon $h$ when $S^h\sim\rho$ and policy
$\pi=(\pi^1,\dots,\pi^H)$ is followed from horizon $h$ onward; similarly,
$d^{h\to i,\pi}_\rho(\tilde s):=\mathbb{E}_{S^h\sim\rho}[d^{h\to i,\pi}_{S^h}(\tilde s)]$.
We initialize the PMD iterates with any full support policy, i.e.,
$\pi_0^h(a\mid s)>0$ for every $h\in[H]$, $s\in\mathcal S$, and
$a\in\mathcal A$; the KL-based PMD update preserves full support for finite
step sizes, since it takes the multiplicative form \eqref{eq:NPG}.

A key technical difficulty in the finite-horizon setting is that, for $h>1$,
the horizon-$h$ visitation distribution induced by a policy $\pi$ generally
admits no policy-independent lower bound. To address this issue, we introduce an auxiliary objective that assigns
policy-independent baseline mass to every horizon, analogous to the
discounted infinite-horizon setting \citep{xiao2022convergence}.
\begin{defi}
\label{def:global_multistart_J}
Let $H_0\sim\mathrm{Unif}\{1,\dots,H\}$ be a random starting horizon, and let
$V^{\pi,H_0}(\rho)$ be the return from reinitializing the process at horizon
$H_0$ with $S^{H_0}\sim\rho$ and following policy $\pi$ thereafter. The global
multi-start objective is
\begin{equation}
\label{eq:global_multistart_J}
J(\pi)
\;:=\;
\mathbb E_{H_0\sim \mathrm{Unif}\{1,\dots,H\}}\!\big[V^{\pi,H_0}(\rho)\big]
\;=\;
\frac{1}{H}\sum_{h_0=1}^H V^{\pi,h_0}(\rho).
\end{equation}
\end{defi}
The optimality gap in the auxiliary objective $J(\pi)$ controls the
horizon-$1$ optimality gap of the original objective up to a factor of $H$.
\begin{lem}
\label{lem:standard_gap_vs_J_gap}
Let $\rho\in\Delta(\mathcal S)$ and let $\pi^\star$ be an optimal
nonstationary policy. Then for any policy $\pi$,
\begin{equation}
\label{eq:standard_gap_vs_J_gap}
V^{\pi^\star,1}(\rho)-V^{\pi,1}(\rho)
\;\le\;
H\big(J(\pi^\star)-J(\pi)\big).
\end{equation}
\end{lem}
To analyze the multi-start objective, we average the visitation distributions
over all admissible starting horizons $h_0\in\{1,\dots,i\}$, which yields the
global multi-start horizon-$i$ visitation measure\footnote{For any nonnegative
finite measure $\mu$ on $\mathcal S$, we write
$\mathbb E_{S\sim\mu}[f(S)]:=\sum_{s\in\mathcal S}\mu(s)f(s)$, even when
$\mu(\mathcal S)\ne1$. In particular, $\bar d^{i,\pi}_\rho(\mathcal S)=i/H$.}
\begin{equation}
\label{eq:bar_d_global}
\bar d^{i,\pi}_\rho(\tilde{s})
\;:=\;
\frac{1}{H}\sum_{h_0=1}^i d^{h_0\to i,\pi}_\rho(\tilde{s}),
\qquad i\in[H],\ \tilde{s}\in\mathcal S.
\end{equation}
The next lemma verifies that this construction provides the required
policy-independent lower bound.
\begin{lem}
\label{lem:global_floor}
Fix $\rho\in\Delta(\mathcal S)$ and a policy $\pi$. Then for every $i\in[H]$
and $\tilde{s}\in\mathcal S$,
\begin{equation}
\label{eq:global_floor}
\bar d^{i,\pi}_\rho(\tilde{s})
\;\ge\;
\frac{1}{H}\rho(\tilde{s}).
\end{equation}
\end{lem}
We next derive the performance-difference identity for $J$, the multi-start
analogue of Lemma~\ref{lem:PDL}.
\begin{lem}{(Performance difference for the global multi-start objective)}
\label{lem:PDLJ}
Let $\pi$ and $\pi'$ be nonstationary policies and fix
$\rho\in\Delta(\mathcal S)$. Then
\begin{equation}\label{eq:PDLJ}
J(\pi')-J(\pi)
=
\sum_{i=1}^H
\mathbb E_{S\sim \bar d^{i,{\pi'}}_\rho}\!\left[
\Big\langle Q^{\pi,i}(S,\cdot),\ \pi'^i(\cdot\mid S)-\pi^i(\cdot\mid S)\Big\rangle
\right].
\end{equation}
\end{lem}
Lemma~\ref{lem:PDLJ} expresses the multi-start objective difference in a form
directly compatible with the KL-based PMD update; as a consequence, $J$ is
nondecreasing along the iterates.
\begin{lem}\label{lem:global_multistart_monotone}
Let $\{\pi_t\}_{t\ge 0}$ be generated by the KL-based PMD update
\eqref{eq:PMD} with step sizes $\eta_t>0$. Then, for all $t\ge 0$,
\[
J(\pi_{t+1})-J(\pi_t)
=
\frac{1}{H}\sum_{h_0=1}^H\Big(V^{\pi_{t+1},h_0}(\rho)-V^{\pi_t,h_0}(\rho)\Big)
\;\ge\;0.
\]
\end{lem}
For an optimal policy $\pi^\star$, define the multi-start optimality gap
\begin{equation} \label{eq:delta_with_J}
    \delta_t:=J(\pi^\star)-J(\pi_t);
\end{equation}
Following the monotonicity Lemma~\ref{lem:global_multistart_monotone}, $\delta_{t+1}\le \delta_t$ for
all $t\ge 0$.
\begin{lem}
\label{lem:onestep_multistart_potential}
Fix $t\ge 0$, and let $\pi_{t+1}$ be generated from $\pi_t$ by the KL-based
PMD update \eqref{eq:PMD} with step size $\eta_t>0$. Suppose that
$\pi_t^i(a\mid s)>0$ for every $i\in[H]$, $s\in\mathcal S$, and
$a\in\mathcal A$, so that all KL terms below are finite. Define the
multi-start KL potential
\begin{equation}
\label{eq:Dtstar_global}
D_t^\star
\;:=\;
\sum_{i=1}^H
\mathbb E_{S\sim \bar d^{i,\pi^\star}_\rho}
\Big[
D_{\mathrm{KL}}\!\big(\pi^{\star,i}(\cdot\mid S)\mid\mid\pi_t^i(\cdot\mid S)\big)
\Big].
\end{equation}
Then
\begin{equation}
\label{eq:onestep_multistart_potential}
\sum_{i=1}^H
\mathbb E_{S\sim \bar d^{i,\pi^\star}_\rho}
\Big[
\big\langle Q^{\pi_t,i}(S,\cdot),\ \pi^{\star,i}(\cdot\mid S)-\pi_{t+1}^i(\cdot\mid S)\big\rangle
\Big]
\;\le\;
\frac{1}{\eta_t}\big(D_t^\star-D_{t+1}^\star\big).
\end{equation}
\end{lem}
\begin{defi}
\label{def:vartheta_k_multistart}
For each $t\ge 0$, the per-iteration multi-start distribution mismatch
coefficient is
\begin{equation}
\label{eq:vartheta_k_multistart}
\vartheta_t
\;:=\;
\max_{i\in[H]}
\max_{\tilde{s}:\,\bar d^{i,\pi^\star}_\rho(\tilde{s})>0}
\frac{\bar d^{i,\pi^\star}_\rho(\tilde{s})}{\bar d^{i,\pi_{t+1}}_\rho(\tilde{s})},
\end{equation}
with the convention $\vartheta_t=+\infty$ if
$\bar d^{i,\pi^\star}_\rho(\tilde{s})>0$ but
$\bar d^{i,\pi_{t+1}}_\rho(\tilde{s})=0$ for some pair $(i,\tilde{s})$.
\end{defi}
By Lemma~\ref{lem:global_floor}, if $\rho(\tilde{s})>0$ then
$\bar d^{i,\pi_{t+1}}_\rho(\tilde{s})\ge \rho(\tilde{s})/H$ for all $i$, so
$\vartheta_t$ is finite whenever $\bar d^{i,\pi^\star}_\rho(\tilde{s})>0$
implies $\rho(\tilde{s})>0$.
\begin{prop}
\label{prop:key_multistart}
Fix $\rho\in\Delta(\mathcal S)$ and an optimal nonstationary policy
$\pi^\star$. Let $\{\pi_t\}_{t\ge0}$ be generated from a full support initial
policy by the KL-based PMD update \eqref{eq:PMD}, with step sizes satisfying
$0<\eta_t<\infty$ for every $t\ge0$. Then, for every $t\ge0$ such that
$\vartheta_t<\infty$,
\begin{equation}
\label{eq:key_multistart}
\vartheta_t(\delta_{t+1}-\delta_t)+\delta_t
\le
\frac{1}{\eta_t}
\bigl(D_t^\star-D_{t+1}^\star\bigr).
\end{equation}
\end{prop}
We further define the multi-start distribution mismatch coefficient
\begin{equation}
\label{eq:vartheta_def_repeat}
\vartheta_\rho
:=
\begin{cases}
\displaystyle
H
\max_{i\in[H]}
\max_{s:\rho(s)>0}
\frac{\bar d^{i,\pi^\star}_\rho(s)}{\rho(s)},
&
\begin{array}{l}
\text{if }
\bar d^{i,\pi^\star}_\rho(s)>0
\Longrightarrow
\rho(s)>0\\[-1pt]
\text{for every }(i,s)\in[H]\times\mathcal S,
\end{array}
\\[1.4em]
+\infty,
&
\text{otherwise}.
\end{cases}
\end{equation}
Thus, $\vartheta_\rho$ is finite if and only if, for every $i\in[H]$, the
support of $\bar d^{i,\pi^\star}_\rho$ is contained in the support of $\rho$;
in particular, $\vartheta_\rho$ is finite whenever $\rho$ has full support on
$\mathcal S$.
\begin{lem}
\label{lem:vartheta_t_le_vartheta_rho}
Fix $\rho\in\Delta(\mathcal S)$ and let $\vartheta_\rho$ be as defined in \eqref{eq:vartheta_def_repeat}. Then for every $t\ge 0$,
\begin{equation}
\label{eq:vartheta_t_le_vartheta_rho}
\vartheta_t \;\le\; \vartheta_\rho.
\end{equation}
\end{lem}
\begin{thm}
\label{thm:geo_multistart}
Fix $\rho\in\Delta(\mathcal S)$ and consider the global multi-start objective
$J(\pi)$ in \eqref{eq:global_multistart_J}. Let $\{\pi_t\}_{t\ge0}$ be generated from a full support initial policy
$\pi_0$ by the KL-based PMD update~\eqref{eq:PMD} with step sizes
$\{\eta_t\}_{t\ge 0}$, and let $\pi^\star$ be an optimal nonstationary policy.
Suppose $\vartheta_\rho\in(1,\infty)$ and the step sizes satisfy $\eta_0>0$
and
\begin{equation}
\label{eq:eta_growth_thm}
\eta_{t+1}\ \ge\ \frac{\vartheta_\rho}{\vartheta_\rho-1}\,\eta_t,
\qquad t=0,1,2,\dots
\end{equation}
Then,
\begin{equation}
\label{eq:linear_rate_J_multistart}
V^{\pi^\star,1}(\rho)-V^{\pi_t,1}(\rho)
\ \le\
H\Big(1-\frac{1}{\vartheta_\rho}\Big)^t
\left(
\delta_0+\frac{1}{\eta_0(\vartheta_\rho-1)}D_0^\star
\right),
\qquad \forall t\ge 0.
\end{equation}
\end{thm}
Note that, unlike Theorem~\ref{thm:GC}, which holds for all horizons
$h\in[H]$, the bound \eqref{eq:linear_rate_J_multistart} controls the
suboptimality at the initial horizon $h=1$ only. This is a consequence of the
multi-start construction, which is essential for the geometric convergence
argument; the reduction to the horizon-$1$ gap is carried out through
Lemma~\ref{lem:standard_gap_vs_J_gap}.

\paragraph{Discussion:}
The role of the random start horizon $H_0$ is to create a policy-independent
baseline mass at every horizon, and its uniform distribution is not arbitrary.
To examine this, consider the weighted generalization $\mathbb P(H_0=h)=m_h$ with
$m_h>0$ and $\sum_{h=1}^H m_h=1$, the weighted objective
$J_m(\pi):=\sum_{h=1}^H m_h\,V^{\pi,h}(\rho)$, and the weighted multi-start
visitation measure
\[
\bar d^{\,i,\pi}_{\rho,m}(s)\;:=\;\sum_{h_0=1}^{i} m_{h_0}\,
d^{h_0\to i,\pi}_\rho(s),
\qquad i\in[H],\ s\in\mathcal S.
\]
The summand $h_0=i$ equals $m_i\,\rho(s)$ for every policy $\pi$, since
initializing at horizon $i$ involves no transitions; as the remaining
summands are nonnegative, we obtain the policy-independent floor
\[
\bar d^{\,i,\pi}_{\rho,m}(s)\;\ge\; m_i\,\rho(s),
\qquad \forall i\in[H],\ \forall s\in\mathcal S,
\]
which is the only point in the analysis where the distribution of $H_0$
provides policy-independent control of the mismatch. Consequently, the
floor-based mismatch bound degrades with $1/m_{\min}$, where
$m_{\min}:=\min_{i\in[H]}m_i$. On the other hand, by horizon-wise optimality
of $\pi^\star$,
\[
J_m(\pi^\star)-J_m(\pi)
\;\ge\;
m_1\big(V^{\pi^\star,1}(\rho)-V^{\pi,1}(\rho)\big),
\]
so the reduction to the horizon-$1$ suboptimality gap incurs the factor
$1/m_1$. Non-uniform weighting therefore trades these two worst-case bounds against
each other: increasing $m_1$ beyond $1/H$ tightens the horizon-$1$ reduction,
but since $\sum_{i=1}^H m_i=1$, it forces some other weight below $1/H$ and
thereby loosens the floor at that horizon. The uniform choice $m_i=1/H$,
which recovers $J$, $\bar d^{\,i,\pi}_\rho$, and $\vartheta_\rho$, is the
unique maximizer of $m_{\min}$ over the probability simplex; it treats all
starting horizons symmetrically and prevents any single horizon from becoming
a bottleneck through an arbitrarily small weight. We emphasize that this is a
worst-case robustness rationale: it does not assert that uniform weighting
minimizes the realized mismatch ratios for every MDP.

\subsection{Horizon-only Robust Step Size Schedule} \label{subsec:robust_increasing}

Throughout this subsection, we assume $H\ge2$. The case $H=1$ is discussed at the end.

\begin{lem}(Horizon lower bound on $\vartheta_\rho$)
\label{lem:vartheta_rho_ge_H}
Suppose that $\vartheta_\rho<\infty$, where $\vartheta_\rho$ is defined in
\eqref{eq:vartheta_def_repeat}. Then $\vartheta_\rho\ge H$. Consequently,
\begin{equation}
\label{eq:ratio_vartheta_by_H}
\frac{\vartheta_\rho}{\vartheta_\rho-1}
\le
\frac{H}{H-1}.
\end{equation}
\end{lem}

\begin{rem}(Upper bounds on $\vartheta_\rho$)
\label{rem:vartheta_upper_bound}
Lemma~\ref{lem:vartheta_rho_ge_H} gives a lower bound on
$\vartheta_\rho$ (see Appendix~\ref{app:linrobust}). For an upper bound, observe that
\[
\bar d^{i,\pi^\star}_\rho(s)
\le
\sum_{\tilde s\in\mathcal S}
\bar d^{i,\pi^\star}_\rho(\tilde s)
=
\frac{i}{H}
\le 1.
\]
Definition~\eqref{eq:vartheta_def_repeat} therefore gives
\[
\vartheta_\rho
\le
\frac{H}{\min_{s:\,\rho(s)>0}\rho(s)}.
\]

Moreover, $\rho$ is used only in the analysis and is not an input to the PMD
update \eqref{eq:PMD}, which is applied independently to every $(h,s)$. Thus,
the same policy sequence can be analyzed using
$\rho=\mathrm{Unif}(\mathcal S)$. For this choice, the support condition in
\eqref{eq:vartheta_def_repeat} holds automatically and
$\vartheta_\rho\le H|\mathcal S|$. A bound under any initial distribution
$\mu$ then follows from the uniform-$\rho$ bound with an additional factor
\[
\max_{s\in\mathcal S}\frac{\mu(s)}{\rho(s)}
\le
|\mathcal S|,
\]
because the per-state gaps
$V^{\pi^\star,1}(s)-V^{\pi_t,1}(s)$ are nonnegative. Thus, $\vartheta_\rho$ has a polynomial upper bound in $H$ and
$|\mathcal S|$. Under the uniform choice of $\rho$, the resulting worst-case
contraction bound depends on $|\mathcal S|$. Obtaining sharper problem-dependent bounds on
$\vartheta_\rho$ remains open.
\end{rem}

\begin{cor}(Horizon-only robust step size schedule)
\label{cor:horizon_only_stepsize}
Let $\{\pi_t\}_{t\ge0}$ be generated from a full support initial policy by the
KL-based PMD update \eqref{eq:PMD}, and let $\pi^\star$ be an optimal
nonstationary policy. Let $\delta_0$ and $D_0^\star$ be defined as in
\eqref{eq:delta_with_J} and \eqref{eq:Dtstar_global}, respectively. Suppose
that $\vartheta_\rho<\infty$. For any $\eta_0>0$, the horizon-only step size
schedule
\begin{equation}
\label{eq:horizon_only_explicit_schedule}
\eta_t
=
\eta_0\left(\frac{H}{H-1}\right)^t,
\qquad t=0,1,2,\ldots,
\end{equation}
satisfies the growth condition \eqref{eq:eta_growth_thm}. Therefore, the
geometric bound \eqref{eq:linear_rate_J_multistart} in
Theorem~\ref{thm:geo_multistart} holds for every $t\ge0$. In particular, if
$\vartheta_\rho=H$, the smallest value permitted by
Lemma~\ref{lem:vartheta_rho_ge_H}, then
\begin{equation}
\label{eq:horizon_only_best_case_bound}
V^{\pi^\star,1}(\rho)-V^{\pi_t,1}(\rho)
\le
H\left(1-\frac{1}{H}\right)^t
\left(
\delta_0+
\frac{D_0^\star}{\eta_0(H-1)}
\right),
\qquad t\ge0.
\end{equation}
\end{cor}

The growth condition \eqref{eq:eta_growth_thm} depends on the generally
unknown coefficient $\vartheta_\rho$ and therefore cannot be implemented
directly. Corollary~\ref{cor:horizon_only_stepsize} removes this dependence.
By Lemma~\ref{lem:vartheta_rho_ge_H}, the multiplier $H/(H-1)$ upper-bounds
$\vartheta_\rho/(\vartheta_\rho-1)$ whenever $\vartheta_\rho<\infty$.
Therefore, the schedule \eqref{eq:horizon_only_explicit_schedule} depends only
on the known horizon length and satisfies the required growth condition
without problem-dependent tuning.

We next compare this result with classical dynamic programming. When the
dynamics are known, backward induction computes $\pi^\star$ exactly in one
pass over the horizons. Thus, Theorem~\ref{thm:geo_multistart}, whose
iteration complexity is of order
$\vartheta_\rho\log(1/\varepsilon)\ge H\log(1/\varepsilon)$ $(\text{where}\; 0<\epsilon<1)$, does not provide
a computational advantage over backward induction. Indeed, since
$\eta_t\to\infty$ under \eqref{eq:horizon_only_explicit_schedule}, each PMD
update tends toward greedy policy improvement with respect to the current
action-values.

Our analysis has a different purpose. NPG/PMD replaces the greedy
maximization in dynamic programming with the smooth multiplicative update
\eqref{eq:NPG}. For a finite step size, this update changes continuously with
$Q^{\pi_t,h}$, while a small change in the action-values can change the greedy
policy abruptly. This smooth policy update motivates TRPO/PPO-style methods
and can be implemented using samples. In infinite-horizon discounted MDPs,
inexact PMD analyses show how errors in estimating the action-values enter
the convergence bound \citep{lan2023policy,xiao2022convergence}. We analyze
the iteration complexity of the exact update in finite-horizon MDPs and leave
the corresponding inexact analysis for future work.

The assumption $H\ge2$ is required only for the mismatch-based growth
condition. When $H=1$, the multi-start objective reduces to the original
objective, $\bar d^{1,\pi}_\rho=\rho$, and $\vartheta_\rho=1$, so the ratio
$\vartheta_\rho/(\vartheta_\rho-1)$ is undefined: in this single-step setting
there is no coupling across horizons; the distribution-mismatch mechanism is
absent, and an increasing step size regime is not the appropriate tool.

\section{Simulation} \label{sec:Simulation}
In this section, we perform simulations to illustrate the sublinear and linear convergence behavior of the NPG algorithm under constant and increasing step sizes, and to validate the predicted scaling behavior of our theoretical bounds. For any horizon $h\in[H]$ and iteration $t$, we measure the suboptimality gap as
\[
\text{Error}(t,h) \;=\; \frac{1}{|\mathcal S|}\sum_{s \in \mathcal{S}} \left( V^{\pi^\star,h}(s) - V^{\pi_t,h}(s) \right).
\]
\subsection{Sublinear convergence with constant step size} \label{subsec:constant}
We construct a randomly generated finite-horizon MDP with $|\mathcal S|=15$, $|\mathcal A|=4$, and $H=7$. 
\begin{figure}[!htb]
\centering

\begin{subfigure}[t]{0.49\columnwidth}
\centering
\begin{tikzpicture}
\begin{semilogyaxis}[
    width=1.00\linewidth,
    height=0.72\linewidth,
    xlabel={Iterations ($t$)},
    ylabel={Error($t,1$)},
    grid=both,
    grid style={dotted},
    legend pos=north east,
    legend cell align={left},
    legend style={font=\scriptsize},
    tick label style={font=\scriptsize},
    label style={font=\scriptsize},
    line join=round,
    line cap=round
]
\addplot+[line width=0.80pt, no marks] table[col sep=comma, x=t, y=actual_sum_gap] {constant_h=1.csv};
\addlegendentry{Error($t,1$)}
\addplot+[line width=0.80pt, dashed, no marks] table[col sep=comma, x=t, y=theory_fixed_state_bound] {constant_h=1.csv};
\addlegendentry{$2H^2/t$}
\end{semilogyaxis}
\end{tikzpicture}
\caption{NPG sublinear convergence at $h=1$.}
\label{fig:constant_h=1}
\end{subfigure}
\hfill
\begin{subfigure}[t]{0.49\linewidth}
\centering
\begin{tikzpicture}
\begin{semilogyaxis}[
    width=\linewidth,
    height=0.72\linewidth,
    xlabel={Horizon ($h$)},
    ylabel={Error$(T,h)$},
    grid=both,
    grid style={dotted},
    legend pos=south west,
    legend cell align={left},
    legend style={font=\scriptsize},
    tick label style={font=\scriptsize},
    label style={font=\scriptsize},
    line join=round,
    line cap=round
]
\addplot+[blue, line width=0.80pt, no marks] 
table[col sep=comma, x=h, y=actual_sum_gap_T_h]
{constant_all_h.csv};
\addlegendentry{Error$(T,h)$}

\addplot+[red, dashed, line width=0.80pt, no marks] 
table[col sep=comma, x=h, y=theory_bound_T_h]
{constant_all_h.csv};
\addlegendentry{$\frac{2(H-h+1)^2}{T}$}

\end{semilogyaxis}
\end{tikzpicture}
\caption{Horizon-wise error at fixed $T=100$ (all $h$).}
\label{fig:tabular_all_h}
\end{subfigure}
\caption{Constant step size NPG on a randomly generated finite-horizon tabular MDP: (a) error vs.\ iteration at $h=1$; (b) error vs.\ all horizons at fixed iteration $T=100$, where $x$-axis is on linear scale and $y$-axis is on logarithmic scale. In the plots, the solid (blue) curve shows the empirical error, and the dashed (red) curve shows the theoretical upper bound.}
\label{fig:tabular-S15A4H7-subfig}
\end{figure}
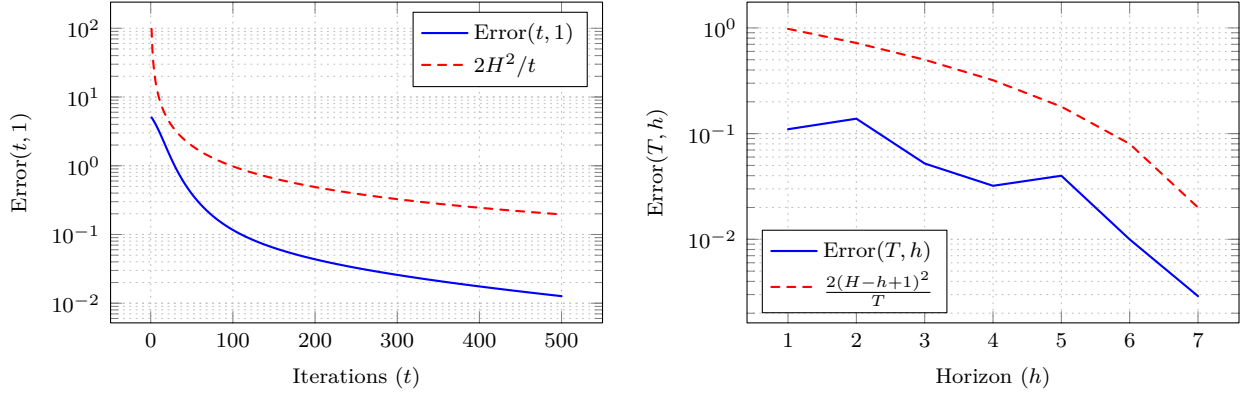

In Figure~\ref{fig:constant_h=1}, we fix $h=1$, initialize $\pi_0$ uniformly, and run NPG for $T=500$ iterations with constant step size $\eta=\log(|\mathcal A|)/H$. We plot the empirical gap together with the theoretical upper bound $2H^2/t$ from Theorem~\ref{thm:GC}. The empirical error decreases monotonically and remains below the theoretical curve, exhibiting the predicted $\mathcal{O}(1/t)$ scaling; the gap between the two is expected, as the bound is a worst-case guarantee.

In Figure~\ref{fig:tabular_all_h}, we instead fix the iteration to $T=100$ and evaluate horizon-wise performance. Starting from the same uniform initialization, we run NPG with the per-horizon step sizes $\eta^h=\log(|\mathcal A|)/(H-h+1)$ of Theorem~\ref{thm:GC} and plot $\text{Error}(T,h)$ against $h$, together with the corresponding bound $2(H-h+1)^2/T$. As $h$ increases, both the empirical gap and the bound shrink, since fewer future rewards remain to be optimized and the bound scales with $(H-h+1)^2$ at fixed $T$.

\subsection{Geometric convergence with increasing step size} \label{subsec:increasing}
We next examine the geometric convergence guarantee of
Theorem~\ref{thm:geo_multistart}.  To obtain an instance with a mismatch
coefficient known in closed form, we consider an MDP with
$|\mathcal{S}|=15$, $|\mathcal{A}|=5$, and $H=7$.
The rewards are drawn independently from the uniform distribution on
$[0,1]$. Across all horizons, we use the same action-independent transition
kernel,
\[
P^h(s' \mid s, a)
\;=\;
0.7 \cdot \mathbf{1}\{s' = s\}
\;+\;
0.3 \cdot \mathbf{1}\{s' = s + 1 \ (\mathrm{mod}\ |\mathcal{S}|)\},
\]
\FloatBarrier
\begin{figure}[!htb]
\centering

    \begin{subfigure}[t]{0.48\linewidth}
        \centering
        \begin{tikzpicture}
            \begin{semilogyaxis}[
                width=\linewidth,
                height=0.70\linewidth,
                xlabel={Iterations $t$},
                ylabel={Error$(t,1)$},
                xmin=0, xmax=50,
                xtick={0,10,20,30,40,50},
                ymin=1e-6, ymax=50,
                ytick={1e-6,1e-4,1e-2,1,10},
                grid=both,
                  grid style={dotted},
    legend pos=south west,
    legend cell align={left},
    legend style={font=\scriptsize},
    tick label style={font=\scriptsize},
    label style={font=\scriptsize},
    line join=round,
    line cap=round
            ]
            \addplot+[blue, line width=0.85pt, no marks]
                table[col sep=comma, x=t, y=horizon1_distributional_gap]
                {increasing_h1_vartheta_H.csv};
            \addlegendentry{Error$(t,1)$}

            \addplot+[red, dashed, line width=0.85pt, no marks]
                table[col sep=comma, x=t, y=best_case_theorem_bound]
                {increasing_h1_vartheta_H.csv};
            \addlegendentry{$(1-\frac{1}{H})^t$}
            \end{semilogyaxis}
        \end{tikzpicture}
        \caption{NPG linear convergence at $h=1$ $(\vartheta_\rho=H)$.}
        \label{fig:increasing_h=1}
    \end{subfigure}
\hspace{0.02\linewidth}%
    \begin{subfigure}[t]{0.48\linewidth}
        \centering
        \begin{tikzpicture}
            \begin{semilogyaxis}[
                width=\linewidth,
                height=0.70\linewidth,
                xlabel={Horizon $(h)$},
                ylabel={Error$(T,h)$},
                xmin=1, xmax=6,
                xtick={1,2,3,4,5,6},
                ymin=1e-16, ymax=1e2,
                ytick={1e-16,1e-12,1e-8,1e-4,1,1e2},
                grid=both,
               grid style={dotted},
    legend pos=south west,
    legend cell align={left},
    legend style={font=\scriptsize},
    tick label style={font=\scriptsize},
    label style={font=\scriptsize},
    line join=round,
    line cap=round
            ]
            \addplot+[blue, line width=0.85pt, no marks]
                table[col sep=comma, x=h, y=actual_error_sum_T_h_plot]
                {increasing_horizonwise_vartheta_m_h_T20.csv};
            \addlegendentry{Error$(T,h)$}

            \addplot+[red, dashed, line width=0.85pt, no marks]
                table[col sep=comma, x=h, y=theory_bound_T_h]
                {increasing_horizonwise_vartheta_m_h_T20.csv};
            \addlegendentry{$(\frac{H-h}{H-h+1})^t$}
            \end{semilogyaxis}
        \end{tikzpicture}
        \caption{Horizon-wise error at fixed $T=20$ $(\text{all}\; h)$.}
        \label{fig:inc_horizonwise_vartheta_m_h_T20}
    \end{subfigure}

\caption{Increasing step size NPG on the structured instance with
$\vartheta_\rho = H$, under the horizon-only robust schedule of
Corollary~\ref{cor:horizon_only_stepsize}: (a) error vs.\ iteration at $h=1$; (b) error vs.\ all horizons at fixed iteration $T=20$. Both plots use linear $x$-axis and log-scale $y$-axis. The solid (blue) curve shows the empirical error, and the dashed (red) curve shows the corresponding geometric reference curve, which is undefined at $h=H$ where $\vartheta_\rho=1$.}
\label{fig:incsteps_row}
\end{figure}
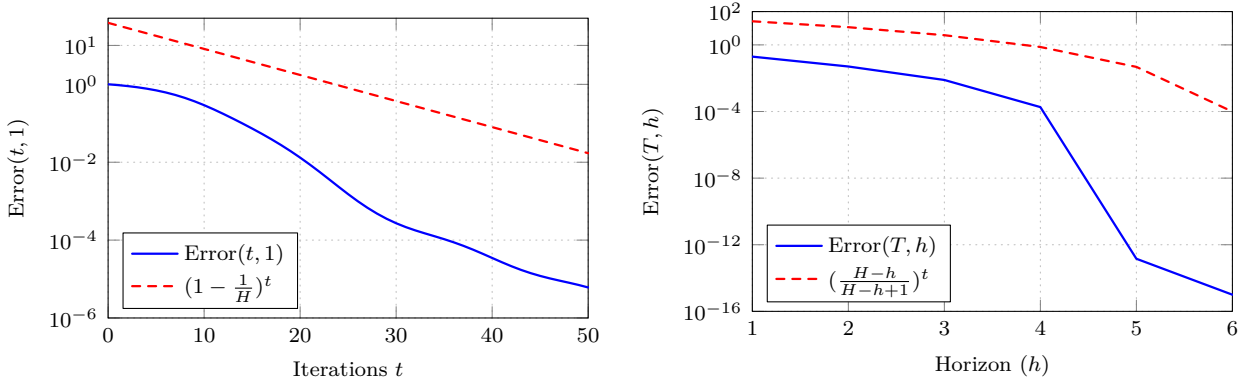
This kernel is doubly stochastic, and hence the uniform initial-state
distribution $\rho$ is invariant under every policy. Consequently, the
optimal multi-start visitation measures satisfy $\bar d^{\,h,\pi^\star}_\rho
=
\frac{h}{H}\rho, h\in[H]$ which implies $\vartheta_\rho=H$. Thus, this instance attains the lower
bound in Lemma~\ref{lem:vartheta_rho_ge_H}.

In Figure~\ref{fig:increasing_h=1}, we consider the objective starting at
$h=1$. We initialize $\pi_0$ uniformly and run NPG for $T=50$ iterations
using the horizon-only schedule from
Corollary~\ref{cor:horizon_only_stepsize}. Because $\vartheta_\rho=H$, the contraction factor in
Theorem~\ref{thm:geo_multistart} becomes $1-1/H$. The empirical error
decreases geometrically and remains below the theoretical bound
evaluated in this instance, consistent with the geometric convergence
guaranteed by the theorem.

In Figure~\ref{fig:inc_horizonwise_vartheta_m_h_T20}, we fix $T=20$ and compare $\text{Error}(T,h)$ across all horizons $h\in[H-1]$. For each starting horizon
$h$, we regard the tail beginning at $h$ as a separate finite-horizon MDP
with effective horizon $H-h+1$. We initialize a separate uniform policy on each tail and use the
corresponding horizon-only robust schedule $\eta^{h}_{t+1}=\frac{H-h+1}{H-h}\,\eta^{h}_t$ of Corollary~\ref{cor:horizon_only_stepsize}. Because every tail retains the same uniform-invariant,
action-independent transition kernel, its multi-start visitation measures
satisfy $\bar d^{\,h,\pi^\star}_{\rho}=\frac{h}{H-h+1}\rho, h\in[H-h+1]$, and its mismatch coefficient is therefore exactly
$\vartheta_{\rho}=H-h+1$. Each data point is directly covered by the best case of Theorem~\ref{thm:geo_multistart} applied
to the corresponding tail MDP. The empirical errors and their theoretical upper bounds become
smaller for later starting horizons, reflecting the shorter remaining
horizon. We omit $h=H$ because the corresponding tail has effective horizon $1$. In this case,
$\vartheta_{\rho}=1$, and the growth factor is undefined;
thus, the increasing step size regime does not apply (see Section~\ref{subsec:robust_increasing}). We repeat both experiments on a
randomly generated MDP, where $\vartheta_\rho$ is evaluated directly from
the multi-start visitation measures and
exceeds its horizon-based lower bound in the generated instance, confirming
that the guarantee is not an artifact of the structured instance (see Appendix~\ref{app:more_simulations}).

\section{Conclusion}\label{sec:conclusion}
In this paper, we have presented the finite-time analysis for exact NPG in
finite-horizon MDPs. For tabular MDPs, we have proved that NPG with a constant step size
achieves the sublinear rate $\mathcal{O}(H^2/t)$. The same rate holds for
linear MDPs under an exact population-projection oracle with a full support
projection distribution, where linear Q-NPG induces the same policy-space
update as tabular NPG. For tabular MDPs, we have established geometric convergence of the horizon-$1$
optimality gap under increasing step sizes, for both the problem-dependent
schedule and the horizon-only robust schedule. 

Regarding future directions, an important next step is to develop sample-based
finite-time guarantees for these results. Another direction is to establish lower bounds that
determine whether the $H^2/t$ rate in Theorem~\ref{thm:GC} and the $H$-factor losses in the multi-start reduction are necessary.

\bibliography{main}
\bibliographystyle{plainnat}

\clearpage
\appendix

\begin{center}
    {\LARGE\bfseries Appendix}
\end{center}

We recall that for a finite-horizon MDP, the state visitation distribution captures the probability of visiting state $\tilde{s}$ at horizon $i$ given that the process starts from $s$ at horizon $h$ and follows policy $\pi$ thereafter in Eq.~\eqref{eq:SVD}. This distribution satisfies the following basic properties:
\begin{itemize}
     \item Initialization: Since the process must be in the state $s$ at the horizon $h$, 
    \[
    d^{h \to h, \pi}_s(\tilde{s}) = \mathbbm{1}\{s=\tilde{s}\}.
    \]
    \item Normalization: At every horizon $i \in [h,H]$,
    \[
    \sum_{\tilde{s}} d^{h \to i, \pi}_s(\tilde{s}) \;=\; 1, \qquad \forall i \ge h.
    \]
\end{itemize}
\section{Sublinear Convergence: Supporting Lemmas and Proofs} \label{app:sublin}
\subsection{MDP Setting Proofs} \label{app:Tabularlem}
In this appendix, we provide detailed algebraic derivations and proofs that support Theorem~\ref{thm:GC}. First, we give the detailed proof of Performance Difference Lemma~\ref{lem:PDL}. 

\begin{proof}[Proof of Lemma \ref{lem:PDL}]
Using the telescoping argument together with the tower property of conditional expectations, we derive
\begin{align*}
V^{\pi,h}(s) & - V^{\pi',h}(s)
=
\E\!\left[\sum_{i=h}^{H} R^i(S^i,A^i)\mid
S^h=s,\ A^j\sim \pi^j(\cdot\mid S^j),\ \forall j=h,\ldots,H
\right]
- V^{\pi',h}(s) \\
&=
\E\!\left[\sum_{i=h}^{H} R^i(S^i,A^i) - V^{\pi',h}(S^h)\mid
S^h=s,\ A^j\sim \pi^j(\cdot\mid S^j)
\right] \\
&\overset{(a)}{=}
\E\!\left[\sum_{i=h}^{H}
\Big(
R^i(S^i,A^i) + V^{\pi',i+1}(S^{i+1}) - V^{\pi',i}(S^i)
\Big)
\mid
S^h=s,\ A^j\sim \pi^j(\cdot\mid S^j)
\right] \\
&\overset{(b)}{=}
\E\!\left[\sum_{i=h}^{H}
\Big(
R^i(S^i,A^i) + \E\big[V^{\pi',i+1}(S^{i+1})\mid S^i,A^i\big] - V^{\pi',i}(S^i)
\Big)
\mid
S^h=s,\ A^j\sim \pi^j(\cdot\mid S^j)
\right] \\
&\overset{(c)}{=}
\E\!\left[\sum_{i=h}^{H} A^{\pi',i}(S^i,A^i)\mid
S^h=s,\ A^j\sim \pi^j(\cdot\mid S^j)
\right] \\
&=
\sum_{i=h}^{H}\sum_{\tilde s,\tilde a}
\mathbb{P}\left(S^i=\tilde s,\, A^i=\tilde a \mid
S^h=s,\ A^j\sim \pi^j(\cdot\mid S^j),\ \forall j=h,\ldots,i
\right)
A^{\pi',i}(\tilde s,\tilde a) \\
&\overset{(d)}{=}
\sum_{i=h}^{H}\sum_{\tilde s,\tilde a}
\mathbb{P}\left(S^i=\tilde s \mid
S^h=s,\ A^j\sim \pi^j(\cdot\mid S^j),\ \forall j=h,\ldots,i-1
\right)
\pi^i(\tilde a\mid \tilde s)\,
A^{\pi',i}(\tilde s,\tilde a) \\
&\overset{(e)}{=}
\sum_{i=h}^{H}\sum_{\tilde s,\tilde a}
d^{h\to i,\pi}_s(\tilde s)\,\pi^i(\tilde a\mid \tilde s)\,A^{\pi',i}(\tilde s,\tilde a).
\end{align*}
Here (a) uses the telescoping identity
$\sum_{i=h}^{H}\big(V^{\pi',i+1}(S^{i+1})-V^{\pi',i}(S^i)\big)
= -V^{\pi',h}(S^h)+V^{\pi',H+1}(S^{H+1})$
together with the convention $V^{\pi',H+1}\equiv 0$,
(b) applies the tower property of conditional expectations,
(c) uses
$A^{\pi',i}(s,a)=Q^{\pi',i}(s,a)-V^{\pi',i}(s)$ with
$Q^{\pi',i}(s,a)=R^i(s,a)+\E\big[V^{\pi',i+1}(S^{i+1})\mid S^i=s,A^i=a\big]$,
(d) factorizes the joint law of $(S^i,A^i)$ as
$\mathbb{P}(S^i=\tilde s)\,\pi^i(\tilde a\mid\tilde s)$, and
(e) substitutes the definition of the state visitation distribution in Eq.~\eqref{eq:SVD}.
\end{proof}

\begin{lem}\label{lemma:Vk}
Fix a horizon $h\in[H]$. Let $\{\pi_t\}_{t=0}^{T}$ be a sequence of policies such that
\[
V^{\pi_{t+1},h}(s)\;\ge\;V^{\pi_t,h}(s)\quad\text{for all }s\in\mathcal S\text{ and }t=0,\dots,T-1,
\]
and let $\pi^\star$ be an optimal policy, i.e., $V^{\pi^\star,h}(s)\ge V^{\pi,h}(s)$ for all $s\in\mathcal S$ and all policies $\pi$.
Then, for every $t\le T$ and every $s\in\mathcal S$,
\[
V^{\pi^\star,h}(s)-V^{\pi_t,h}(s)\;\ge\;V^{\pi^\star,h}(s)-V^{\pi_T,h}(s).
\]
\end{lem}
\begin{proof}
Applying the monotonicity assumption repeatedly along $t,t+1,\dots,T$ gives $V^{\pi_T,h}(s)\ge V^{\pi_t,h}(s)$. Multiplying $(-1)$ in both sides of the inequality and adding $V^{\pi^\star,h}(s)$ to both sides yields
\[
V^{\pi^\star,h}(s)-V^{\pi_t,h}(s)\;\ge\;V^{\pi^\star,h}(s)-V^{\pi_T,h}(s);
\]
that is, the suboptimality gap is nonincreasing along the sequence $\{\pi_t\}_{t=0}^{T}$, which proves the claim.
\end{proof}

\begin{lem}{(Nonnegativity of the KL divergence)}\label{lem:KL}
For any distributions $\pi(\cdot\mid s)$ and $\pi'(\cdot\mid s)$ over a finite action set
$\mathcal A$, the KL divergence satisfies
\[
D_{\mathrm{KL}}\big(\pi(\cdot\mid s)\mid\mid\pi'(\cdot\mid s)\big)\ge 0.
\]
If $\pi(\cdot\mid s)$ is not absolutely continuous with respect to $\pi'(\cdot\mid s)$, then
the KL divergence is defined as $+\infty$, and the inequality is immediate.
\end{lem}

\begin{proof}
If $\pi(\cdot\mid s)$ is not absolutely continuous with respect to $\pi'(\cdot\mid s)$, then
there exists an action $a\in\mathcal A$ such that
\[
\pi(a\mid s)>0
\qquad\text{and}\qquad
\pi'(a\mid s)=0.
\]
In this case,
\[
D_{\mathrm{KL}}\big(\pi(\cdot\mid s)\mid\mid\pi'(\cdot\mid s)\big)=+\infty,
\]
and the claim is immediate.

Otherwise, assume
\[
\pi(\cdot\mid s)\ll \pi'(\cdot\mid s),
\]
meaning that $\pi(a\mid s)>0$ implies $\pi'(a\mid s)>0$ for every $a\in\mathcal A$. Hence all
ratios below are well-defined on the support of $\pi(\cdot\mid s)$. We also use the standard
convention
\[
0\log\frac{0}{q}=0,
\qquad q\ge0,
\]
so actions with $\pi(a\mid s)=0$ contribute zero to the KL divergence.

Let
\[
\operatorname{supp}(\pi(\cdot\mid s))
:=
\{a\in\mathcal A:\pi(a\mid s)>0\}.
\]
From the definition of the KL divergence,
\begin{align*}
D_{\mathrm{KL}}\big(\pi(\cdot\mid s)\mid\mid\pi'(\cdot\mid s)\big)
&=
\sum_{a\in \operatorname{supp}(\pi(\cdot\mid s))}
\pi(a\mid s)
\log\frac{\pi(a\mid s)}{\pi'(a\mid s)} \\
&=
\sum_{a\in \operatorname{supp}(\pi(\cdot\mid s))}
\pi(a\mid s)
\left(
-\log\frac{\pi'(a\mid s)}{\pi(a\mid s)}
\right).
\end{align*}
Since $x\mapsto-\log x$ is convex, Jensen's inequality gives
\begin{align*}
D_{\mathrm{KL}}\big(\pi(\cdot\mid s)\mid\mid\pi'(\cdot\mid s)\big)
&\ge
-\log\left(
\sum_{a\in \operatorname{supp}(\pi(\cdot\mid s))}
\pi(a\mid s)
\frac{\pi'(a\mid s)}{\pi(a\mid s)}
\right) \\
&=
-\log\left(
\sum_{a\in \operatorname{supp}(\pi(\cdot\mid s))}
\pi'(a\mid s)
\right).
\end{align*}
Since
\[
\sum_{a\in \operatorname{supp}(\pi(\cdot\mid s))}
\pi'(a\mid s)
\le
\sum_{a\in\mathcal A}\pi'(a\mid s)
=
1,
\]
and $x\mapsto-\log x$ is decreasing on $(0,\infty)$, it follows that
\[
-\log\left(
\sum_{a\in \operatorname{supp}(\pi(\cdot\mid s))}
\pi'(a\mid s)
\right)
\ge
-\log(1)
=
0.
\]
Therefore,
\[
D_{\mathrm{KL}}\big(\pi(\cdot\mid s)\mid\mid\pi'(\cdot\mid s)\big)\ge0,
\]
which proves the claim.
\end{proof}

\begin{lem}\label{lemma:logZi}
Fix an iteration $t\ge 0$, a horizon $h\in[H]$, a state $s\in\mathcal S$, and any step size $\eta>0$. Then the following inequality holds:
\[
\frac{1}{\eta}\,\log Z_t^h(s) - V^{\pi_t,h}(s) \;\geq\; 0,
\quad\text{where}\quad
V^{\pi_t,h}(s)=\sum_{a'\in\mathcal A}\pi_t^h(a'\mid s)\,Q^{\pi_t,h}(s,a').
\]
\end{lem}
\begin{proof}
Substituting the definition of the normalization constant $Z_t^h(s)$ from Eq.~\eqref{eq:NPG} and applying Jensen's inequality to the concave function $\log$, we obtain
\[
\log Z^h_t(s)
=\log\!\left[\,\sum_{a'} \pi^h_t(a'\mid s)\,\exp\!\big(\eta\, Q^{\pi_t,h}(s,a')\big)\right]
\;\geq\; \sum_{a'} \pi^h_t(a'\mid s)\,\eta\, Q^{\pi_t,h}(s,a')
\;=\; \eta\, V^{\pi_t,h}(s).
\]
Dividing both sides by $\eta>0$ yields
\[
\frac{1}{\eta}\log Z^h_t(s) \;\geq\; V^{\pi_t,h}(s),
\]
which is the claim.
\end{proof} 
\begin{lem} \label{lem:L}
Fix an iteration $t\ge0$, a horizon $h\in[H]$, a state $s\in\mathcal S$, and a step size
$\eta>0$. For each $l\in\{h,\ldots,H\}$, define
\[
F_l
:=
\sum_{\tilde s\in\mathcal S}
d_s^{h\to l,\pi_{t+1}}(\tilde s)
\left[
\frac{1}{\eta}\log Z_t^l(\tilde s)-V^{\pi_t,l}(\tilde s)
\right].
\]
Then
\[
\sum_{l=h}^H F_l
\ge
\frac{1}{\eta}\log Z_t^h(s)-V^{\pi_t,h}(s).
\]
\end{lem}

\begin{proof}
We consider,
\begin{align*} 
    \sum^H_{l=h} F_l &  = \sum^H_{l=h} \sum_{\tilde{s}} d^{{h} \to {l}, \pi_{t+1}}_s (\tilde{s}) \left[\frac{1}{\eta}  \log{Z^{l}_t(\tilde{s})}-V^{\pi_t, l} (\tilde{s})\right] \\
    & \geq \sum_{\tilde{s}} d^{{h} \to {h}, \pi_{t+1}}_s (\tilde{s}) \left[\frac{1}{\eta}  \log{Z^{h}_t(\tilde{s})}-V^{\pi_t, h} (\tilde{s})\right] + \sum^H_{l=h+1} \sum_{\tilde{s}} d^{{h} \to {l}, \pi_{t+1}}_s (\tilde{s}) \left[\frac{1}{\eta}  \log{Z^{l}_t(\tilde{s})}-V^{\pi_t, l} (\tilde{s})\right] \\
    & = F_h + \sum^H_{l=h+1} {F_l}.
\end{align*}

For case, $l=h$:

\begin{align*} 
    F_h & = \sum_{\tilde{s}} d^{{h} \to {h}, \pi_{t+1}}_s (\tilde{s}) \left[\frac{1}{\eta}  \log{Z^{h}_t(\tilde{s})}-V^{\pi_t, h} (\tilde{s})\right]  = \frac{1}{\eta}  \log{Z^{h}_t(s)}-V^{\pi_t, h} (s) . 
\end{align*}

For case, $l >  h$:
\vspace{-5pt}
\begin{align*} 
    \sum^H_{l=h+1} {F_l} & = \sum^H_{l=h+1} \sum_{\tilde{s}} d^{{h} \to {l}, \pi_{t+1}}_s (\tilde{s}) \left[\frac{1}{\eta}  \log{Z^{l}_t(\tilde{s})}-V^{\pi_t, l} (\tilde{s}) \right] \geq 0  . \tag{\text Lemma \ref{lemma:logZi}}
\end{align*}
Thus, considering both of the cases, we can conclude as follows,
\[\sum^H_{l=h} F_l \geq \frac{1}{\eta}  \log{Z^{h}_t(s) }-V^{\pi_t, h} (s).\]
\end{proof}
Next, we present the lemma for the improved lower bound, which serves as a crucial component throughout.

\paragraph{Proof of Lemma~\ref{lem:ILB}}~\newline
Since the rewards are bounded and
the horizon is finite, $Q^{\pi_t,i}(s,a)$ is finite for every
$i\in[H]$, $s\in\mathcal S$, and $a\in\mathcal A$. Therefore, the
multiplicative update~\eqref{eq:NPG} and the full support assumption
imply
\[
\pi_{t+1}^i(a\mid s)>0,
\qquad
\forall i\in[H],\ s\in\mathcal S,\ a\in\mathcal A.
\]
Consequently, all logarithmic ratios below are well-defined, and
\[
D_{\mathrm{KL}}\!\left(
\pi_{t+1}^i(\cdot\mid s)
\mid\mid
\pi_t^i(\cdot\mid s)
\right)
<\infty.
\]
Applying Lemma~\ref{lem:PDL} with
$(\pi,\pi')=(\pi_{t+1},\pi_t)$ gives
\begin{align*} 
V^{\pi_{t+1},h}(s)  - V^{\pi_t,h}(s)
&= \sum^H_{l=h}  \sum_{\tilde{s},\tilde{a}} d^{{h} \to {l}, \pi_{t+1}}_s (\tilde{s}) \ \pi^{l
}_{t+1} (\tilde{a}\mid \tilde{s})   A^{\pi_t,l}(\tilde{s}, \tilde{a})  \\
& = \sum^H_{l=h} 
 \sum_{\tilde{s},\tilde{a}} d^{{h} \to {l}, \pi_{t+1}}_s (\tilde{s}) \ \pi^{l}_{t+1} (\tilde{a}\mid \tilde{s})  \left[Q^{\pi_t,{l}} (\tilde{s}, \tilde{a}) -V^{\pi_t,l} (\tilde{s}) \right]  \\    
& =\sum^H_{l=h}  \sum_{\tilde{s},\tilde{a}} d^{{h} \to {l}, \pi_{t+1}}_s (\tilde{s}) \ \pi^{l}_{t+1} (\tilde{a}\mid \tilde{s}) \bigg[ \frac{1}{\eta}  \log \frac{\pi^{l}_{t+1}(\tilde{a}\mid \tilde{s}){Z^{l}_t(\tilde{s})}}{\pi^{l}_t (\tilde{a}\mid \tilde{s})}-V^{\pi_t,l} (\tilde{s}) \bigg] \tag{\text{Eq.}\eqref{eq:NPG}}\\
& = \sum^H_{l=h} \sum_{\tilde{s}} d^{{h} \to {l}, \pi_{t+1}}_s (\tilde{s}) \bigg[\frac{1}{\eta}  \ D_{\mathrm{KL}}(\pi^{l}_{t+1} (\cdot \mid \tilde{s}) \mid\mid \pi^{l}_{t}(\cdot \mid\tilde{s})) +\frac{1}{\eta}  \log{Z^{l}_t(\tilde{s})}-V^{\pi_t,l} (\tilde{s})\bigg] \\
& \geq \sum^H_{l=h} \sum_{\tilde{s}} d^{{h} \to {l}, \pi_{t+1}}_s (\tilde{s}) \bigg[\frac{1}{\eta} \log{Z^{l}_t(\tilde{s})} -V^{\pi_t,l} (\tilde{s}) \bigg]  \tag{\text Lemma {\ref{lem:KL}}} \\
& \geq \frac{1}{\eta}  \log{Z^{h}_t(s) }-V^{\pi_t, h} (s)  \tag{\text Lemma {\ref{lem:L}}} \\
& \geq 0. \  \tag{\text Lemma {\ref{lemma:logZi}}}
\end{align*}

Hence,
\[V^{\pi_{t+1},h}(s)-V^{\pi_t,h}(s) \geq \frac{1}{\eta}  \log{Z^{h}_t(s)}-V^{\pi_t, h} (s) \geq 0.\]
\hfill
$\square$

\begin{lem}\label{lem:m_0}
Let $\pi_0^i(\cdot\mid s)=\mathrm{Unif}(\mathcal A)$ for all $i\in[H]$ and $s\in\mathcal S$. Then for every $i\in[H]$ and every $s\in\mathcal S$,
\[
D_{\mathrm{KL}}\big(\pi^{\star,i}(\cdot\mid s)\mid\mid\pi^i_0(\cdot\mid s)\big) \;\le\; \log|\mathcal A|,
\]
and consequently, for any $j\in\{0,\dots,H-1\}$ and any $s\in\mathcal S$,
\[
m_0^{H-j\to H}(s)\;=\;\sum^H_{i=H-j}\sum_{\tilde s} d^{H-j\to i,\pi^\star}_s(\tilde s)\,D_{\mathrm{KL}}\big(\pi^{\star,i}(\cdot\mid \tilde s)\mid\mid\pi^i_0(\cdot\mid \tilde s)\big)\;\le\;(j+1)\log|\mathcal A|,
\]
where the second inequality uses $\sum_{\tilde s} d^{H-j\to i,\pi^\star}_s(\tilde s)=1$.
\end{lem}
\begin{proof}
Applying the definition of the KL divergence and substituting the uniform initial policy $\pi_0^i(a\mid s)=\frac{1}{|\mathcal A|}$, we have
\begin{align*}
D_{\mathrm{KL}}\big(\pi^{\star,i}(\cdot\mid s)\mid\mid\pi^i_0(\cdot\mid s)\big)
&= \sum_{a\in\mathcal A} \pi^{\star,i}(a\mid s)\,\log\frac{\pi^{\star,i}(a\mid s)}{\pi^i_0(a\mid s)} \\
&= \sum_{a\in\mathcal A} \pi^{\star,i}(a\mid s)\,\log\big[\pi^{\star,i}(a\mid s)\cdot|\mathcal A|\big] \\
&= \sum_{a\in\mathcal A} \pi^{\star,i}(a\mid s)\log \pi^{\star,i}(a\mid s) \;+\; \log|\mathcal A| \\
&\le \log|\mathcal A|,
\end{align*}
where the last step uses $\sum_{a}\pi^{\star,i}(a\mid s)\log\pi^{\star,i}(a\mid s)\le 0$, since each term is nonpositive, and $\sum_a \pi^{\star,i}(a\mid s)=1$ in the second summand.

For the second claim, fix $j\in\{0,\dots,H-1\}$ and $s\in\mathcal S$. Applying the per-state bound above to every $\tilde s$ in the inner sum,
\begin{align*}
m_0^{H-j\to H}(s)
&= \sum^H_{i=H-j}\ \sum_{\tilde s} d^{H-j\to i,\pi^\star}_s(\tilde s)\,
D_{\mathrm{KL}}\big(\pi^{\star,i}(\cdot\mid \tilde s)\mid\mid\pi^i_0(\cdot\mid \tilde s)\big) \\
&\le \sum^H_{i=H-j}\ \sum_{\tilde s} d^{H-j\to i,\pi^\star}_s(\tilde s)\,\log|\mathcal A| \\
&= \sum^H_{i=H-j} \log|\mathcal A|
\tag{$\textstyle\sum_{\tilde s} d^{H-j\to i,\pi^\star}_s(\tilde s)=1$} \\
&= (j+1)\log|\mathcal A|,
\end{align*}
since the sum over $i$ ranges over $H-(H-j)+1=j+1$ terms. This completes the proof.
\end{proof}
\begin{lem}\label{lem:nk}
For any $j\in\{0,\ldots,H-1\}$, $s\in\mathcal S$, and $t\ge0$, define
\[
n_t^{H-j\to H}(s)
:=
\sum_{i=H-j}^H n_t^i(s)
=
\sum_{i=H-j}^H
\sum_{\tilde s\in\mathcal S}
d_s^{H-j\to i,\pi^\star}(\tilde s)\,
V^{\pi_t,i}(\tilde s).
\]
Then
\[
n_t^{H-j\to H}(s)
\le
\frac{(j+1)(j+2)}{2}
\le
(j+1)^2.
\]
\end{lem}

\begin{proof}
Since the rewards take values in $[0,1]$, the value function satisfies
\[
0
\le
V^{\pi_t,i}(\tilde s)
\le
H-i+1
\]
for every $t\ge0$, $i\in[H]$, and $\tilde s\in\mathcal S$. Therefore,
\begin{align*}
n_t^{H-j\to H}(s)
&=
\sum_{i=H-j}^{H}
\sum_{\tilde s\in\mathcal S}
d_s^{H-j\to i,\pi^\star}(\tilde s)\,
V^{\pi_t,i}(\tilde s)\\
&\le
\sum_{i=H-j}^{H}
\sum_{\tilde s\in\mathcal S}
d_s^{H-j\to i,\pi^\star}(\tilde s)\,
(H-i+1)\\
&=
\sum_{i=H-j}^{H}(H-i+1)\\
&=
\frac{(j+1)(j+2)}{2}\\
&\le
(j+1)^2.
\end{align*}
The third line uses
\[
\sum_{\tilde s\in\mathcal S}
d_s^{H-j\to i,\pi^\star}(\tilde s)
=
1,
\qquad i=H-j,\ldots,H,
\]
and the final inequality follows from $j\ge0$. This proves the claim.
\end{proof}

\paragraph{{Proof of Theorem \ref{thm:GC}}}~\newline 
Fix an integer $T\ge1$. The uniform initial policy $\pi_0$ has full
support. We next verify that the NPG update preserves this property.
Suppose that $\pi_t$ has full support. Since the rewards are bounded
and the horizon is finite, $Q^{\pi_t,i}(s,a)$ is finite for every
$i\in[H]$, $s\in\mathcal S$, and $a\in\mathcal A$. Therefore,
\[
\pi_{t+1}^i(a\mid s)
=
\frac{
\pi_t^i(a\mid s)
\exp\!\left(\eta Q^{\pi_t,i}(s,a)\right)
}{
Z_t^i(s)
}
>
0.
\]
Thus, by induction, $\pi_t$ has full support for every
$t=0,\ldots,T$. Consequently, Lemma~\ref{lem:ILB} applies at every
iteration, and all KL terms appearing below are finite.

Applying the performance difference Lemma \ref{lem:PDL}, we obtain
\begin{align*} 
V^{\pi^\star,{H-j}}(s) & - V^{\pi_t,{H-j}}(s) =\sum^H_{i=H-j} \sum_{\tilde{s},\tilde{a}} d^{{H-j} \to i, \pi^\star}_s (\tilde{s})  \pi^{\star,i} (\tilde{a}\mid \tilde{s}) A^{\pi_t,i}(\tilde{s}, \tilde{a}) \\
& = \sum^H_{i=H-j} \sum_{\tilde{s},\tilde{a}} d^{{H-j} \to i, \pi^\star}_s (\tilde{s}) \pi^{\star,i} (\tilde{a}\mid \tilde{s}) \bigg[Q^{\pi_t,i} (\tilde{s}, \tilde{a})-V^{\pi_t,i} (\tilde{s}) \bigg]   \\
& = \sum^H_{i=H-j}  \sum_{\tilde{s},\tilde{a}} d^{{H-j} \to i, \pi^\star}_s (\tilde{s}) \pi^{\star,i} (\tilde{a}\mid \tilde{s}) \bigg[ \frac{1}{\eta}\log \frac{\pi^{i}_{t+1}(\tilde{a}\mid \tilde{s}){Z^{i}_t(\tilde{s})}}{\pi^{i}_t (\tilde{a}\mid \tilde{s})} \ -  V^{\pi_t,i} (\tilde{s}) \bigg]   \tag{Eq.(\ref{eq:NPG})} \\
& =\sum^H_{i=H-j} \sum_{\tilde{s}} d^{{H-j} \to i, \pi^\star}_s (\tilde{s}) \bigg[\frac{1}{\eta} \left[D_{\mathrm{KL}}(\pi^{\star,i}(\cdot \mid\tilde{s}) \mid\mid \pi^{i}_t(\cdot \mid\tilde{s}))-D_{\mathrm{KL}} (\pi^{\star,i}(\cdot \mid\tilde{s}) \mid\mid \pi^{i}_{t+1} (\cdot \mid\tilde{s})) \right] \\ &\quad \quad \quad \quad \quad \quad \quad \quad \quad \quad \quad \quad +\frac{1}{\eta} \log{Z^{i}_t(\tilde{s})}-V^{\pi_t,i} (\tilde{s})\bigg] \\
& \leq \sum^H_{i=H-j} \sum_{\tilde{s}} d^{{H-j} \to i, \pi^\star}_s (\tilde{s})  \bigg[\frac{1}{\eta} \left[D_{\mathrm{KL}}(\pi^{\star,i}(\cdot \mid\tilde{s}) \mid\mid \pi^{i}_t(\cdot \mid\tilde{s}))-D_{\mathrm{KL}}(\pi^{\star,i}(\cdot \mid\tilde{s}) \mid\mid \pi^{i}_{t+1} (\cdot \mid\tilde{s})) \right] \\ &\quad \quad \quad \quad \quad  \quad \quad \quad \quad  \quad \quad \quad + V^{\pi_{t+1},i}(\tilde{s}) - V^{\pi_{t},i}(\tilde{s})\bigg] \tag{Lemma \ref{lem:ILB}}\\
& = \sum^H_{i=H-j} \bigg[ \frac{1}{\eta} (m^{i}_t (s)- {m}^{i}_{t+1} (s)) + n_{t+1}^i (s)-n_{t}^i (s) \bigg]\\
&= \frac{1}{\eta} ( m_t^{H-j\to H} (s)-m_{t+1}^{H-j\to H} (s) ) + n_{t+1}^{H-j\to H} (s)- n_{t}^{H-j\to H} (s),
\end{align*}
where for $j \in \{0,1,\cdots, H-1\}$, we assume that
\begin{align*}
m_t^{H-j\to H} (s) & =\sum^H_{i=H-j} m^i_t (s) = \sum^H_{i=H-j} 
 \sum_{\tilde{s} \in \mathcal{S}} d^{{H-j} \to i, \pi^\star}_s (\tilde{s})  D_{\mathrm{KL}}(\pi^{\star,i} (\cdot \mid\tilde{s}) \mid\mid\pi^i_t (\cdot \mid\tilde{s})),
\end{align*}
and
\begin{align*} 
n_{t}^{H-j\to H} (s) & = \sum^H_{i=H-j} n^i_t (s) = \sum^H_{i=H-j} \sum_{\tilde{s} \in \mathcal{S}} d^{{H-j} \to i, \pi^\star}_s (\tilde{s})V^{\pi_t,i}(\tilde{s}).
\end{align*}
Applying the above improvement bound and telescopic argument, for all states $s \in \mathcal{S}$, we have  
\begin{align*}
\frac{1}{T}\sum^{T-1}_{t=0} & \bigg[ V^{\pi^\star,{H-j}}(s) - V^{\pi_t,{H-j}}(s)\bigg] \\
& \leq  \frac{1}{T} \sum^{T-1}_{t=0} 
 \bigg[ \frac{1}{\eta} \big( m_t^{H-j\to H} (s) -m_{t+1}^{H-j\to H} (s) \big) +  n_{t+1}^{H-j\to H} (s)-n_{t}^{H-j\to H}(s)\bigg] \\
 & = \frac{1}{T} \bigg[ \frac{1}{\eta} \big( m_0^{H-j\to H} (s) -m_T^{H-j\to H} (s) \big) + n_T^{H-j\to H} (s)-n_0^{H-j\to H}(s)\bigg] \\
 & \overset{(a)}{\leq} \frac{1}{T} \left[\frac{1}{\eta}  \sum^H_{i=H-j} {m}^{i}_0 (s)  + n_{T}^{H-j\to H} (s) \right]\\
&  \leq \frac{1}{T} \left[\frac{1}{\eta} (j+1) \log|\mathcal A| + (j+1)^2 \right],  \tag{\text{Lemmas \ref{lem:m_0} and \ref{lem:nk}}} 
\end{align*}
where in (a), we have used that $m_T^{H-j\to H} (s) \geq 0$ and $n_0^{H-j\to H}(s) \geq 0$.

Applying Lemma \ref{lemma:Vk} we obtain,
\begin{equation} \nonumber
V^{\pi^\star,{H-j}}(s) - V^{\pi_T,{H-j}}(s) \leq \frac{(j+1)\log|\mathcal A|}{\eta T} +  \frac{(j+1)^2} {T}.
\end{equation}
Substituting $j=H-h$, we obtain
\begin{equation} \nonumber
   V^{\pi^\star,h}(s) - V^{\pi_T,h}(s) \leq \frac{(H-h+1)\log|\mathcal A|}{\eta T} +  \frac{(H-h+1)^2} {T}  ,
\end{equation}
which concludes the result.
\hfill
$\square$

\subsection{Linear MDP Proofs} \label{app:Linlem}
In this appendix, whenever $w^{\pi_t,h}$ denotes the output of the full support
exact population projection oracle, we use the pointwise identity
\eqref{eq:linear_oracle_exact_at_t}.

\begin{lem}{\citep{jin2020provably}} \label{lem:Qlin}
Under Assumption~\ref{ass:LinMDP}, for any policy $\pi$, there exist vectors
$\{w^{\pi,h}\}_{h=1}^H\subset\mathbb R^d$ such that, for every
$(s,a,h)\in\mathcal S\times\mathcal A\times[H]$,
\[
    Q^{\pi,h}(s,a)
    =
    \langle w^{\pi,h},\phi(s,a)\rangle .
\]
\end{lem}

\begin{proof}
Fix a policy $\pi$ and a horizon $h\in[H]$. By the Bellman equation and the linear MDP
structure in Assumption~\ref{ass:LinMDP}, for every $(s,a)\in\mathcal S\times\mathcal A$,
\begin{align*}
Q^{\pi,h}(s,a)
&=
R^h(s,a)
+
(P^h V^{\pi,h+1})(s,a) \\
&=
\langle \phi(s,a),\zeta^h\rangle
+
\sum_{s'\in\mathcal S}
P^h(s'\mid s,a)V^{\pi,h+1}(s') \\
&=
\langle \phi(s,a),\zeta^h\rangle
+
\sum_{s'\in\mathcal S}
\langle \phi(s,a),\omega^h(s')\rangle
V^{\pi,h+1}(s') \\
&=
\left\langle
\phi(s,a),
\zeta^h+\sum_{s'\in\mathcal S}\omega^h(s')V^{\pi,h+1}(s')
\right\rangle .
\end{align*}
Therefore, defining
\[
w^{\pi,h}
:=
\zeta^h+\sum_{s'\in\mathcal S}\omega^h(s')V^{\pi,h+1}(s')
\]
gives
\[
Q^{\pi,h}(s,a)
=
\langle w^{\pi,h},\phi(s,a)\rangle,
\qquad
\forall (s,a)\in\mathcal S\times\mathcal A .
\]
This proves the claim.
\end{proof}

For the auxiliary lemmas below, define for each iteration $t$, horizon $h$, and state $s$,
\[
\widetilde Z_t^h(s)
:=
\sum_{a'\in\mathcal A}
\pi_t^h(a'\mid s)
\exp\!\left(
\eta\,\langle w^{\pi_t,h},\phi(s,a')\rangle
\right).
\]
\begin{lem}\label{lemma:logZilin}
Fix an iteration $t\ge0$, a horizon $h\in[H]$, a state $s\in\mathcal S$, and a step size
$\eta>0$. Under an exact population-projection oracle with a full support projection distribution,
\[
\frac{1}{\eta}\log \widetilde Z_t^h(s)-V^{\pi_t,h}(s)\ge0.
\]
\end{lem}

\begin{proof}
Using the definition of $\widetilde Z_t^h(s)$ and applying Jensen's inequality to the concave
function $\log$, we obtain
\begin{align*}
\log \widetilde Z_t^h(s)
&=
\log\!\left[
\sum_{a'\in\mathcal A}
\pi_t^h(a'\mid s)
\exp\!\left(
\eta\,\langle w^{\pi_t,h},\phi(s,a')\rangle
\right)
\right] \\
&\ge
\sum_{a'\in\mathcal A}
\pi_t^h(a'\mid s)
\log
\exp\!\left(
\eta\,\langle w^{\pi_t,h},\phi(s,a')\rangle
\right) \\
&=
\eta
\sum_{a'\in\mathcal A}
\pi_t^h(a'\mid s)
\langle w^{\pi_t,h},\phi(s,a')\rangle \\
&=
\eta
\sum_{a'\in\mathcal A}
\pi_t^h(a'\mid s)
Q^{\pi_t,h}(s,a') \\
&=
\eta V^{\pi_t,h}(s).
\end{align*}
The fourth line uses the pointwise identity~\eqref{eq:linear_oracle_exact_at_t}, and the final line uses
the definition of the value function. Dividing by $\eta>0$ gives
\[
\frac{1}{\eta}\log \widetilde Z_t^h(s)\ge V^{\pi_t,h}(s),
\]
which proves the claim.
\end{proof}

The next lemma aggregates the nonnegative local terms across future horizons.

\begin{lem}\label{lem:Llin}
Fix an iteration $t\ge0$, a horizon $h\in[H]$, a state $s\in\mathcal S$, and a step size
$\eta>0$. For each $l\in\{h,\ldots,H\}$, define
\[
F_l'
:=
\sum_{\tilde s\in\mathcal S}
d_s^{h\to l,\pi_{t+1}}(\tilde s)
\left[
\frac{1}{\eta}\log \widetilde Z_t^l(\tilde s)
-
V^{\pi_t,l}(\tilde s)
\right].
\]
Then
\[
\sum_{l=h}^H F_l'
\ge
\frac{1}{\eta}\log \widetilde Z_t^h(s)-V^{\pi_t,h}(s).
\]
\end{lem}

\begin{proof}
By Lemma~\ref{lemma:logZilin}, for every $l\in\{h,\ldots,H\}$ and every
$\tilde s\in\mathcal S$,
\[
\frac{1}{\eta}\log \widetilde Z_t^l(\tilde s)
-
V^{\pi_t,l}(\tilde s)
\ge0.
\]
Since $d_s^{h\to l,\pi_{t+1}}(\tilde s)\ge0$, it follows that $F_l'\ge0$ for every
$l\in\{h,\ldots,H\}$.

For $l=h$, the visitation distribution satisfies
\[
d_s^{h\to h,\pi_{t+1}}(\tilde s)
=
\mathbbm{1}\{\tilde s=s\}.
\]
Therefore,
\begin{align*}
F_h'
&=
\sum_{\tilde s\in\mathcal S}
\mathbbm{1}\{\tilde s=s\}
\left[
\frac{1}{\eta}\log \widetilde Z_t^h(\tilde s)
-
V^{\pi_t,h}(\tilde s)
\right] \\
&=
\frac{1}{\eta}\log \widetilde Z_t^h(s)-V^{\pi_t,h}(s).
\end{align*}
Combining the nonnegativity of $F_l'$ for all $l>h$ with the expression for $F_h'$ gives
\[
\sum_{l=h}^H F_l'
=
F_h'+\sum_{l=h+1}^H F_l'
\ge
F_h'
=
\frac{1}{\eta}\log \widetilde Z_t^h(s)-V^{\pi_t,h}(s).
\]
This proves the claim.
\end{proof}

\begin{lem}\label{lem:ILB_linear_oracle}
(Improvement Lower Bound in the Linear MDP Setting)
Fix an iteration $t\ge0$ and a horizon $h\in[H]$. Suppose the Q-NPG update
\eqref{eq:Q-NPG-update-oracle} is run under an exact population-projection oracle with a full support projection distribution
of Section~\ref{sec:linear}. Then, for every $s\in\mathcal S$,
\[
V^{\pi_{t+1},h}(s)-V^{\pi_t,h}(s)
\ge
\frac{1}{\eta}\log \widetilde Z_t^h(s)-V^{\pi_t,h}(s)
\ge0.
\]
\end{lem}

\begin{proof}
By the softmax parametrization~\eqref{eq:policy_parametrization} and the Q-NPG parameter
update~\eqref{eq:Q-NPG-update-oracle}, for every $l\in[H]$,
$\tilde s\in\mathcal S$, and $\tilde a\in\mathcal A$, the induced policy update satisfies
\[
\pi_{t+1}^l(\tilde a\mid\tilde s)
=
\frac{
\pi_t^l(\tilde a\mid\tilde s)
\exp\!\left(
\eta\,\langle w^{\pi_t,l},\phi(\tilde s,\tilde a)\rangle
\right)
}{
\widetilde Z_t^l(\tilde s)
}.
\]
Since the softmax parametrization assigns strictly positive probability to every action, we may
rearrange the preceding display as
\[
\langle w^{\pi_t,l},\phi(\tilde s,\tilde a)\rangle
=
\frac{1}{\eta}
\log
\frac{
\pi_{t+1}^l(\tilde a\mid\tilde s)\,\widetilde Z_t^l(\tilde s)
}{
\pi_t^l(\tilde a\mid\tilde s)
}.
\]
Using the pointwise identity~\eqref{eq:linear_oracle_exact_at_t}, this gives
\begin{equation}
\label{eq:linear_Q_log_ratio_appendix}
Q^{\pi_t,l}(\tilde s,\tilde a)
=
\frac{1}{\eta}
\log
\frac{
\pi_{t+1}^l(\tilde a\mid\tilde s)\,\widetilde Z_t^l(\tilde s)
}{
\pi_t^l(\tilde a\mid\tilde s)
}.
\end{equation}

Applying the performance difference lemma~\ref{lem:PDL} with
$(\pi,\pi')=(\pi_{t+1},\pi_t)$ yields
\begin{align*}
V^{\pi_{t+1},h}(s)-V^{\pi_t,h}(s)
&=
\sum_{l=h}^H
\sum_{\tilde s,\tilde a}
d_s^{h\to l,\pi_{t+1}}(\tilde s)\,
\pi_{t+1}^l(\tilde a\mid\tilde s)\,
A^{\pi_t,l}(\tilde s,\tilde a) \\
&=
\sum_{l=h}^H
\sum_{\tilde s,\tilde a}
d_s^{h\to l,\pi_{t+1}}(\tilde s)\,
\pi_{t+1}^l(\tilde a\mid\tilde s)
\left[
Q^{\pi_t,l}(\tilde s,\tilde a)-V^{\pi_t,l}(\tilde s)
\right] \\
&=
\sum_{l=h}^H
\sum_{\tilde s,\tilde a}
d_s^{h\to l,\pi_{t+1}}(\tilde s)\,
\pi_{t+1}^l(\tilde a\mid\tilde s)
\left[
\frac{1}{\eta}
\log
\frac{
\pi_{t+1}^l(\tilde a\mid\tilde s)\,\widetilde Z_t^l(\tilde s)
}{
\pi_t^l(\tilde a\mid\tilde s)
}
-
V^{\pi_t,l}(\tilde s)
\right] \\
&=
\sum_{l=h}^H
\sum_{\tilde s}
d_s^{h\to l,\pi_{t+1}}(\tilde s)
\left[
\frac{1}{\eta}
D_{\mathrm{KL}}
\!\left(
\pi_{t+1}^l(\cdot\mid\tilde s)
\mid\mid
\pi_t^l(\cdot\mid\tilde s)
\right)
+
\frac{1}{\eta}\log \widetilde Z_t^l(\tilde s)
-
V^{\pi_t,l}(\tilde s)
\right] \\
&\ge
\sum_{l=h}^H
\sum_{\tilde s}
d_s^{h\to l,\pi_{t+1}}(\tilde s)
\left[
\frac{1}{\eta}\log \widetilde Z_t^l(\tilde s)
-
V^{\pi_t,l}(\tilde s)
\right] \\
&\ge
\frac{1}{\eta}\log \widetilde Z_t^h(s)-V^{\pi_t,h}(s) \\
&\ge0.
\end{align*}
The first inequality uses the nonnegativity of the KL divergence. The second inequality follows
from Lemma~\ref{lem:Llin}, and the final inequality follows from Lemma~\ref{lemma:logZilin}.
Therefore,
\[
V^{\pi_{t+1},h}(s)-V^{\pi_t,h}(s)
\ge
\frac{1}{\eta}\log \widetilde Z_t^h(s)-V^{\pi_t,h}(s)
\ge0,
\]
which proves the claim.
\end{proof}
The preceding lemmas show that, under an exact population-projection oracle with a full support projection distribution, the
linear MDP setting inherits the same one-step improvement structure as the tabular setting.
Since the projection is exact, no approximation, projection, or regression-error term appears in
the lower bound. We now prove Proposition~\ref{prop:linMDP} directly from the induced
policy space identity~\eqref{eq:linear_equals_tabular_update}. 

\paragraph{Proof of Proposition~\ref{prop:linMDP}}~\newline
By the pointwise identity~\eqref{eq:linear_oracle_exact_at_t} and the induced update
identity~\eqref{eq:linear_equals_tabular_update}, the policy sequence generated by the Q-NPG
parameter update~\eqref{eq:Q-NPG-update-oracle} satisfies exactly the tabular NPG update
\eqref{eq:NPG} with constant step size $\eta$.

Moreover, $\theta_0^h=\mathbf 0$ for every $h\in[H]$ implies, by the softmax
parametrization~\eqref{eq:policy_parametrization}, that
\[
\pi_0^h(\cdot\mid s)=\mathrm{Unif}(\mathcal A),
\qquad
\forall h\in[H],\ s\in\mathcal S .
\]
Therefore, the induced policy sequence satisfies precisely the assumptions of
Theorem~\ref{thm:GC}. Applying Theorem~\ref{thm:GC}, for every integer $T\ge1$, for all horizons $h\in[H]$ and states $s\in\mathcal S$, we obtain
\[
V^{\pi^\star,h}(s)-V^{\pi_T,h}(s)
\le
\frac{(H-h+1)\log|\mathcal A|}{\eta T}
+
\frac{(H-h+1)^2}{T}.
\]
This completes the proof.
\hfill
$\square$

\section{Geometric Convergence Proofs} \label{app:lin}
\begin{proof} [Proof of Lemma~\ref{lem:standard_gap_vs_J_gap}]
By the definition of the global multi-start objective $J(\cdot)$, we have
\begin{align*}
H\big(J(\pi^\star)-J(\pi)\big)
&=
\sum_{h_0=1}^H\Big(V^{\pi^\star,h_0}(\rho)-V^{\pi,h_0}(\rho)\Big).
\end{align*}
Now, we use horizon-wise optimality of $\pi^\star$, for each horizon $h_0\in\{1,\dots,H\}$ and each state $s\in\mathcal S$ to have,
$V^{\pi^\star,h_0}(s)\ge V^{\pi,h_0}(s)$. Taking expectation over $S\sim\rho$ yields
$V^{\pi^\star,h_0}(\rho)-V^{\pi,h_0}(\rho)\ge 0$, for every $h_0$. Therefore, each term in the above sum is nonnegative.
In particular, the sum of these nonnegative terms dominates the $h_0=1$ term:
\begin{align*} V^{\pi^\star,1}(\rho)-V^{\pi,1}(\rho) 
 & \;\le\;
\sum_{h_0=1}^H\Big(V^{\pi^\star,h_0}(\rho)-V^{\pi,h_0}(\rho)\Big) \\
& = H\big(J(\pi^\star)-J(\pi)\big)
\end{align*}
which completes the proof.
\end{proof}

\begin{proof}  [Proof of Lemma~\ref{lem:global_floor}]
Fix $i\in\{1,\dots,H\}$ and $\tilde{s}\in\mathcal S$.
By construction, the term in \eqref{eq:bar_d_global} corresponding to $h_0=i$ is
\[
d^{i\to i,\pi}_\rho(\tilde{s})
=
\mathbb P_{S^i\sim\rho}(S^i=\tilde{s})
=
\rho(\tilde{s}),
\]
which holds because the process is initialized at horizon $i$ with $S^i\sim\rho$, and the case $h_0=i$ involves no transitions.\\
Since each summand $d^{h_0\to i,\pi}_\rho(\tilde{s})$ is a probability and hence nonnegative, we obtain
\[
\bar d^{i,\pi}_\rho(\tilde{s})
=
\frac{1}{H}\sum_{h_0=1}^{i} d^{h_0\to i,\pi}_\rho(\tilde{s})
\;\ge\;
\frac{1}{H}\, d^{i\to i,\pi}_\rho(\tilde{s})
=
\frac{1}{H}\rho(\tilde{s}),
\]
which proves \eqref{eq:global_floor}.
\end{proof}

\begin{proof}[Proof of Lemma~\ref{lem:PDLJ}]
Starting from the definition of the global multi-start objective, $J(\cdot)$:
\[
J(\pi')-J(\pi)
=
\frac{1}{H}\sum_{h_0=1}^H\Big(V^{\pi',h_0}(\rho)-V^{\pi,h_0}(\rho)\Big).
\]

Fix a start horizon $h_0\in[H]$. Apply the finite-horizon performance difference lemma (Lemma~\ref{lem:PDL}) at horizon $h_0$ with $\pi$ and $\pi'$ in the appropriate order gives 
\[
V^{\pi',h_0}(\rho)-V^{\pi,h_0}(\rho)
=
\sum_{i=h_0}^H\ \sum_{\tilde{s}\in\mathcal S}
d^{h_0\to i,\pi'}_\rho(\tilde{s})\,
\Big\langle Q^{\pi,i}(\tilde{s},\cdot),\ \pi'^i(\cdot\mid \tilde{s})-\pi^i(\cdot\mid \tilde{s})\Big\rangle.
\]
Substituting this identity into the average over $h_0$ and exchanging the finite sums, we obtain
\begin{align*}
J(\pi')-J(\pi)
&=
\frac{1}{H}\sum_{h_0=1}^H \sum_{i=h_0}^H \sum_{\tilde{s}\in\mathcal S}
d^{h_0\to i,\pi'}_\rho(\tilde{s})\,
\Big\langle Q^{\pi,i}(\tilde{s},\cdot),\ \pi'^i(\cdot\mid \tilde{s})-\pi^i(\cdot\mid \tilde{s})\Big\rangle \\
&=
\sum_{i=1}^H \sum_{\tilde{s}\in\mathcal S}
\left(\frac{1}{H}\sum_{h_0=1}^i d^{h_0\to i,\pi'}_\rho(\tilde{s})\right)
\Big\langle Q^{\pi,i}(\tilde{s},\cdot),\ \pi'^i(\cdot\mid \tilde{s})-\pi^i(\cdot\mid \tilde{s})\Big\rangle.
\end{align*}

By the definition of the global multi-start horizon-$i$ visitation measure in \eqref{eq:bar_d_global},
\[
\frac{1}{H}\sum_{h_0=1}^i d^{h_0\to i,\pi'}_\rho(\tilde s)
=
\bar d^{i,\pi'}_\rho(\tilde s).
\]
Therefore,
\[
J(\pi')-J(\pi)
=
\sum_{i=1}^H \sum_{\tilde s\in\mathcal S}
\bar d^{i,\pi'}_\rho(\tilde s)\,
\Big\langle Q^{\pi,i}(\tilde s,\cdot),\ \pi'^i(\cdot\mid \tilde s)-\pi^i(\cdot\mid \tilde s)\Big\rangle,
\]
which is exactly the expectation form in \eqref{eq:PDLJ}.
\end{proof}

\begin{proof}[Proof of Lemma~\ref{lem:global_multistart_monotone}]
We first establish a pointwise (local) nonnegativity property of the KL-based PMD update.

Fix any horizon $h\in[H]$ and state $s\in\mathcal S$. By optimality of $\pi_{t+1}^h(\cdot\mid s)$ in the update \eqref{eq:PMD} and the three-point inequality for KL-based mirror descent \citep[see, e.g.,][]{lan2023policy,xiao2022convergence}, applied to the linear objective $p\mapsto \eta_t\langle Q^{\pi_t,h}(s,\cdot),p\rangle$, for any $p\in\Delta(\mathcal A)$ we have
\begin{align*}
\eta_t \big\langle Q^{\pi_t,h}(s,\cdot), \pi_{t+1}^h(\cdot\mid s) \big\rangle
- D_{\mathrm{KL}}\!\left(\pi_{t+1}^h(\cdot\mid s)\mid\mid\pi_t^h(\cdot\mid s)\right)
\ge& \;
\eta_t \big\langle Q^{\pi_t,h}(s,\cdot), p \big\rangle
- D_{\mathrm{KL}}\!\left(p\mid\mid\pi_t^h(\cdot\mid s)\right) \\
& \qquad \qquad \qquad+
D_{\mathrm{KL}}\!\left(p\mid\mid\pi_{t+1}^h(\cdot\mid s)\right).
\end{align*}
Rearranging and dividing by $\eta_t>0$ gives
\begin{align*}
\big\langle Q^{\pi_t,h}(s,\cdot),\, \pi_{t+1}^h(\cdot\mid s)-p \big\rangle
\;\ge\;&\;
\frac{1}{\eta_t}D_{\mathrm{KL}}\!\left(\pi_{t+1}^h(\cdot\mid s)\mid\mid\pi_t^h(\cdot\mid s)\right)
+\frac{1}{\eta_t}D_{\mathrm{KL}}\!\left(p\mid\mid\pi_{t+1}^h(\cdot\mid s)\right) \\
& \qquad \qquad \qquad \qquad \qquad \qquad \qquad \quad -\;
\frac{1}{\eta_t}D_{\mathrm{KL}}\!\left(p\mid\mid\pi_t^h(\cdot\mid s)\right).
\end{align*}
Choosing $p=\pi_t^h(\cdot\mid s)$ and using
$D_{\mathrm{KL}}(\pi_t^h(\cdot\mid s)\mid\mid\pi_t^h(\cdot\mid s))=0$, we obtain
\begin{align*}
    \big\langle Q^{\pi_t,h}(s,\cdot),\, \pi_{t+1}^h(\cdot\mid s)-\pi_t^h(\cdot\mid s) \big\rangle &
\ge
\frac{1}{\eta_t}D_{\mathrm{KL}}\!\left(\pi_{t+1}^h(\cdot\mid s)\mid\mid\pi_t^h(\cdot\mid s)\right)
+\frac{1}{\eta_t}D_{\mathrm{KL}}\!\left(\pi_t^h(\cdot\mid s)\mid\mid\pi_{t+1}^h(\cdot\mid s)\right) \\ & \ge 0,
\end{align*}
where the last step uses the nonnegativity of the KL divergence (Lemma~\ref{lem:KL}). Thus for every $(h,s)$,
\begin{equation}\label{eq:local_Q_nonneg_J}
\big\langle Q^{\pi_t,h}(s,\cdot),\, \pi_{t+1}^h(\cdot\mid s)-\pi_t^h(\cdot\mid s) \big\rangle
\;\ge\;0.
\end{equation}

Next, fix a starting horizon $h_0\in[H]$. Applying the finite-horizon performance difference lemma (Lemma~\ref{lem:PDL}) at horizon $h_0$ with initial distribution $\rho$ to the pair $(\pi,\pi')=(\pi_{t+1},\pi_t)$, and using $V^{\pi_t,i}(\tilde s)=\langle Q^{\pi_t,i}(\tilde s,\cdot),\pi_t^i(\cdot\mid\tilde s)\rangle$ to rewrite the advantage sum as an inner product of differences, gives
\[
V^{\pi_{t+1},h_0}(\rho)-V^{\pi_t,h_0}(\rho)
=
\sum_{i=h_0}^H\sum_{\tilde s\in\mathcal S}
d^{h_0\to i,\pi_{t+1}}_\rho(\tilde s)\,
\big\langle Q^{\pi_t,i}(\tilde s,\cdot),\,\pi_{t+1}^i(\cdot\mid\tilde s)-\pi_t^i(\cdot\mid\tilde s)\big\rangle.
\]
By \eqref{eq:local_Q_nonneg_J} with $(h,s)=(i,\tilde s)$ and the fact that
$d^{h_0\to i,\pi_{t+1}}_\rho(\tilde s)\ge 0$, every summand is nonnegative, so
\[
V^{\pi_{t+1},h_0}(\rho)-V^{\pi_t,h_0}(\rho)\;\ge\;0,
\qquad \forall h_0\in[H].
\]
Averaging these inequalities over $h_0\in[H]$ yields
\[
J(\pi_{t+1})-J(\pi_t)
=
\frac{1}{H}\sum_{h_0=1}^H\Big(V^{\pi_{t+1},h_0}(\rho)-V^{\pi_t,h_0}(\rho)\Big)
\;\ge\;0,
\]
which proves the monotonicity of the global multi-start objective $J(\cdot)$.
\end{proof}

\begin{proof}[Proof of Lemma~\ref{lem:onestep_multistart_potential}]
We apply the one-step KL mirror-descent inequality \citep[see, e.g.,][]{lan2023policy,xiao2022convergence} at the coordinate $(i,s)$: for any $q\in\Delta(\mathcal A)$,
\begin{equation}
\label{eq:use_onestep}
\big\langle Q^{\pi_t,i}(s,\cdot),\ q(\cdot)-\pi_{t+1}^i(\cdot\mid s)\big\rangle
\;\le\;
\frac{1}{\eta_t}
\Big(
D_{\mathrm{KL}}\!\big(q\mid\mid\pi_t^i(\cdot\mid s)\big)
-
D_{\mathrm{KL}}\!\big(q\mid\mid\pi_{t+1}^i(\cdot\mid s)\big)
\Big).
\end{equation}
Choosing $q(\cdot)=\pi^{\star,i}(\cdot\mid s)$ gives
\begin{align}
\label{eq:use_q_star}
\big\langle Q^{\pi_t,i}(s,\cdot),\ \pi^{\star,i}(\cdot\mid s)-\pi_{t+1}^i(\cdot\mid s)\big\rangle
\;\le\;
\frac{1}{\eta_t}
\Big(
D_{\mathrm{KL}}\!\big(\pi^{\star,i}(\cdot\mid s)\mid\mid\pi_t^i(\cdot\mid s)\big)
-
D_{\mathrm{KL}}\!\big(\pi^{\star,i}(\cdot\mid s)\mid\mid\pi_{t+1}^i(\cdot\mid s)\big)
\Big).
\end{align}
Taking expectation of both sides with respect to $S\sim\bar d^{i,\pi^\star}_\rho$ and summing over
$i=1,\ldots,H$, we obtain
\begin{align*}
\sum_{i=1}^H
\mathbb E_{S\sim \bar d^{i,\pi^\star}_\rho}
\Big[
\big\langle Q^{\pi_t,i}(S,\cdot),\ \pi^{\star,i}(\cdot\mid S)-\pi_{t+1}^i(\cdot\mid S)\big\rangle
\Big]
&\le
\frac{1}{\eta_t}
\sum_{i=1}^H
\Big(
\mathbb E_{S\sim \bar d^{i,\pi^\star}_\rho}
D_{\mathrm{KL}}\!\big(\pi^{\star,i}(\cdot\mid S)\mid\mid\pi_t^i(\cdot\mid S)\big) \\
&\quad\quad -
\mathbb E_{S\sim \bar d^{i,\pi^\star}_\rho}
D_{\mathrm{KL}}\!\big(\pi^{\star,i}(\cdot\mid S)\mid\mid\pi_{t+1}^i(\cdot\mid S)\big)
\Big) \\
&=
\frac{1}{\eta_t}\big(D_t^\star-D_{t+1}^\star\big),
\end{align*}
where the last equality follows from the definition of $D_t^\star$ in \eqref{eq:Dtstar_global}.
This proves \eqref{eq:onestep_multistart_potential}.
\end{proof}

\paragraph{Proof of Proposition~\ref{prop:key_multistart}}~\newline
Fix $t\ge0$ such that $\vartheta_t<\infty$. Because $\pi_0$ has full
support and each step size $\eta_t$ is finite, the multiplicative form
of the KL-based PMD update preserves full support. Consequently, all
KL terms appearing below are finite, and
Lemma~\ref{lem:onestep_multistart_potential} applies.

Applying that lemma gives
\begin{align}
\label{eq:step1_global}
\sum_{i=1}^H
\mathbb E_{S\sim\bar d^{i,\pi^\star}_\rho}
\left[
\left\langle
-Q^{\pi_t,i}(S,\cdot),
\pi_{t+1}^i(\cdot\mid S)-\pi^{\star,i}(\cdot\mid S)
\right\rangle
\right]
\le
\frac{1}{\eta_t}
\bigl(D_t^\star-D_{t+1}^\star\bigr).
\end{align}
For every $i\in[H]$ and $S\in\mathcal S$, decompose the inner product as
\begin{align*}
\left\langle
-Q^{\pi_t,i}(S,\cdot),
\pi_{t+1}^i(\cdot\mid S)-\pi^{\star,i}(\cdot\mid S)
\right\rangle & =
\left\langle
-Q^{\pi_t,i}(S,\cdot),
\pi_{t+1}^i(\cdot\mid S)-\pi_t^i(\cdot\mid S)
\right\rangle
\\
& \qquad \qquad \qquad \qquad +
\left\langle
Q^{\pi_t,i}(S,\cdot),
\pi^{\star,i}(\cdot\mid S)-\pi_t^i(\cdot\mid S)
\right\rangle.
\end{align*}
Accordingly, denote the two resulting sums in the left-hand side
of~\eqref{eq:step1_global} by $(\mathrm I)$ and $(\mathrm{II})$,
respectively.

For term $(\mathrm{II})$, Lemma~\ref{lem:PDLJ}, applied with
$(\pi,\pi')=(\pi_t,\pi^\star)$, gives
\begin{align*}
(\mathrm{II})
&=
\sum_{i=1}^H
\mathbb E_{S\sim\bar d^{i,\pi^\star}_\rho}
\left[
\left\langle
Q^{\pi_t,i}(S,\cdot),
\pi^{\star,i}(\cdot\mid S)-\pi_t^i(\cdot\mid S)
\right\rangle
\right]
\\
&=
J(\pi^\star)-J(\pi_t)
=
\delta_t \tag{From~\eqref{eq:delta_with_J}}.
\end{align*}

For term $(\mathrm I)$, define
\[
f_i(s)
:=
\left\langle
-Q^{\pi_t,i}(s,\cdot),
\pi_{t+1}^i(\cdot\mid s)-\pi_t^i(\cdot\mid s)
\right\rangle.
\]
The local monotonicity property~\eqref{eq:local_Q_nonneg_J} implies
\[
f_i(s)\le0,
\qquad
\forall i\in[H],\ s\in\mathcal S.
\]

If $\bar d^{i,\pi^\star}_\rho(s)>0$, the definition of
$\vartheta_t$ and the assumption $\vartheta_t<\infty$ give
\[
\bar d^{i,\pi^\star}_\rho(s)
\le
\vartheta_t\,
\bar d^{i,\pi_{t+1}}_\rho(s).
\]
If $\bar d^{i,\pi^\star}_\rho(s)=0$, the same inequality holds
trivially because $\vartheta_t\ge0$ and
$\bar d^{i,\pi_{t+1}}_\rho(s)\ge0$. Hence the inequality holds for
every $s\in\mathcal S$. Since $f_i(s)\le0$, multiplying by $f_i(s)$
reverses the inequality:
\[
\bar d^{i,\pi^\star}_\rho(s)f_i(s)
\ge
\vartheta_t\,
\bar d^{i,\pi_{t+1}}_\rho(s)f_i(s).
\]
Summing over $s\in\mathcal S$ and $i\in[H]$ yields
\begin{equation}
\label{eq:I_lower}
(\mathrm I)
\ge
\vartheta_t
\sum_{i=1}^H
\mathbb E_{S\sim\bar d^{i,\pi_{t+1}}_\rho}
\bigl[f_i(S)\bigr].
\end{equation}

Applying Lemma~\ref{lem:PDLJ} with
$(\pi,\pi')=(\pi_t,\pi_{t+1})$ gives
\begin{align*}
\sum_{i=1}^H
\mathbb E_{S\sim\bar d^{i,\pi_{t+1}}_\rho}
\bigl[f_i(S)\bigr]
&=
J(\pi_t)-J(\pi_{t+1})\\
&=
\delta_{t+1}-\delta_t.
\end{align*}
Therefore,
\[
(\mathrm I)
\ge
\vartheta_t(\delta_{t+1}-\delta_t).
\]
Combining this inequality with
$(\mathrm{II})=\delta_t$ and~\eqref{eq:step1_global} gives
\[
\vartheta_t(\delta_{t+1}-\delta_t)+\delta_t
\le
\frac{1}{\eta_t}
\bigl(D_t^\star-D_{t+1}^\star\bigr),
\]
which proves~\eqref{eq:key_multistart}.
\hfill
$\square$

\begin{proof} [Proof of Lemma~\ref{lem:vartheta_t_le_vartheta_rho}]
If $\vartheta_\rho=+\infty$, then \eqref{eq:vartheta_t_le_vartheta_rho} holds trivially. Hence, assume
$\vartheta_\rho<\infty$. By the definition of $\vartheta_\rho$, this implies the support condition
\[
\bar d^{i,\pi^\star}_\rho(s)>0 \ \Longrightarrow\ \rho(s)>0
\qquad \text{for all } i\in[H],\ s\in\mathcal S.
\]

Now fix any $i\in[H]$ and any state $s\in\mathcal S$ such that $\bar d^{i,\pi^\star}_\rho(s)>0$. Then $\rho(s)>0$, and applying Lemma~\ref{lem:global_floor} to the policy $\pi_{t+1}$ yields
\begin{equation}
\label{eq:lower_bar_d_pi_t1}
\bar d^{i,\pi_{t+1}}_\rho(s)\ \ge\ \frac{1}{H}\rho(s).
\end{equation}
Therefore,
\[
\frac{\bar d^{i,\pi^\star}_\rho(s)}{\bar d^{i,\pi_{t+1}}_\rho(s)}
\ \le\
\frac{\bar d^{i,\pi^\star}_\rho(s)}{\rho(s)/H}
=
H\,\frac{\bar d^{i,\pi^\star}_\rho(s)}{\rho(s)}
\ \le\
H\max_{j\in[H]}\max_{x:\rho(x)>0}\frac{\bar d^{j,\pi^\star}_\rho(x)}{\rho(x)}
=
\vartheta_\rho.
\]
Since the inequality holds for every admissible pair $(i,s)$ with $\bar d^{i,\pi^\star}_\rho(s)>0$,
taking the maxima in the definition of $\vartheta_t$ yields $\vartheta_t\le \vartheta_\rho$.
\end{proof}

\paragraph{Proof of Theorem~\ref{thm:geo_multistart}}~\newline
Recall that $\delta_t := J(\pi^\star)-J(\pi_t)$. By Lemma~\ref{lem:global_multistart_monotone}, $\{\delta_t\}_{t\ge 0}$ is nonincreasing, so $\delta_{t+1}-\delta_t\le 0$. Since $\vartheta_\rho<\infty$ and Lemma~$\ref{lem:vartheta_t_le_vartheta_rho}$ gives $\vartheta_t\le\vartheta_\rho$, we have $\vartheta_t<\infty$ for every $t\ge0$. Therefore,
\[
\vartheta_\rho(\delta_{t+1}-\delta_t)+\delta_t
\;\le\;
\vartheta_t(\delta_{t+1}-\delta_t)+\delta_t.
\]
Combining this with Proposition~\ref{prop:key_multistart} yields
\begin{equation}
\label{eq:key_with_vartheta_rho}
\vartheta_\rho(\delta_{t+1}-\delta_t)+\delta_t
\;\le\;
\frac{1}{\eta_t}\big(D_t^\star-D_{t+1}^\star\big).
\end{equation}
Rearranging  \eqref{eq:key_with_vartheta_rho} gives
\[
\vartheta_\rho\delta_{t+1}-(\vartheta_\rho-1)\delta_t
\;\le\;
\frac{1}{\eta_t}D_t^\star-\frac{1}{\eta_t}D_{t+1}^\star.
\]
Moving the $D_{t+1}^\star$ term to the left and dividing by $\vartheta_\rho$ to obtain
\begin{align*}
\delta_{t+1}+\frac{1}{\eta_t\vartheta_\rho}D_{t+1}^\star
& \;\le\;
\Big(1-\frac{1}{\vartheta_\rho}\Big)\delta_t+\frac{1}{\eta_t\vartheta_\rho}D_t^\star \\
& = \Big(1-\frac{1}{\vartheta_\rho}\Big)\delta_t+ \frac{(\vartheta_\rho-1)}{\eta_t\vartheta_\rho(\vartheta_\rho-1)}D_t^\star \\
& = \Big(1-\frac{1}{\vartheta_\rho}\Big)\delta_t+ \Big(1-\frac{1}{\vartheta_\rho}\Big) \frac{1}{\eta_t (\vartheta_\rho-1)}D_t^\star \\
& = \Big(1-\frac{1}{\vartheta_\rho}\Big) \Big[\delta_t+\frac{1}{\eta_t (\vartheta_\rho-1)}D_t^\star \Big].
\end{align*}
If the step size satisfies \eqref{eq:eta_growth_thm}, i.e. $\eta_{t+1}(\vartheta_\rho-1) \geq \eta_t\vartheta_\rho$, then we have,
\[\delta_{t+1}+\frac{1}{\eta_{t+1}(\vartheta_\rho-1)}D_{t+1}^\star \leq \Big(1-\frac{1}{\vartheta_\rho}\Big) \Big[\delta_t+\frac{1}{\eta_t (\vartheta_\rho-1)}D_t^\star \Big] \]
This forms the following recursion, for all $t \geq 0$,
\[\delta_{t}+\frac{1}{\eta_{t}(\vartheta_\rho-1)}D_{t}^\star \leq \Big(1-\frac{1}{\vartheta_\rho}\Big)^t \Big[\delta_0+\frac{1}{\eta_0 (\vartheta_\rho-1)}D_0^\star \Big] \]
Finally, since $D_t^\star\ge 0$ we have $\delta_t \le \delta_t+\frac{1}{\eta_t(\vartheta_\rho-1)}D_t^\star$, hence
\[\delta_t = J(\pi^\star)-J(\pi_t)
\ \le\
\Big(1-\frac{1}{\vartheta_\rho}\Big)^t
\left(
\delta_0+\frac{1}{\eta_0(\vartheta_\rho-1)}D_0^\star
\right),
\qquad \forall t\ge 0.\]
By Lemma~\ref{lem:standard_gap_vs_J_gap}, for any policy $\pi$, start at horizon-$1$, we know
\[
V^{\pi^\star,1}(\rho)-V^{\pi,1}(\rho)
\;\le\;
H\big(J(\pi^\star)-J(\pi)\big).
\]
Apply this inequality with $\pi=\pi_t$ to obtain
\[V^{\pi^\star,1}(\rho)-V^{\pi_t,1}(\rho)
\ \le\
H\Big(1-\frac{1}{\vartheta_\rho}\Big)^t
\left(
\delta_0+\frac{1}{\eta_0(\vartheta_\rho-1)}D_0^\star
\right),
\qquad \forall t\ge 0.\]
This completes the proof.
\hfill
$\square$

\subsection{Further Simulations} \label{app:more_simulations} 
\FloatBarrier
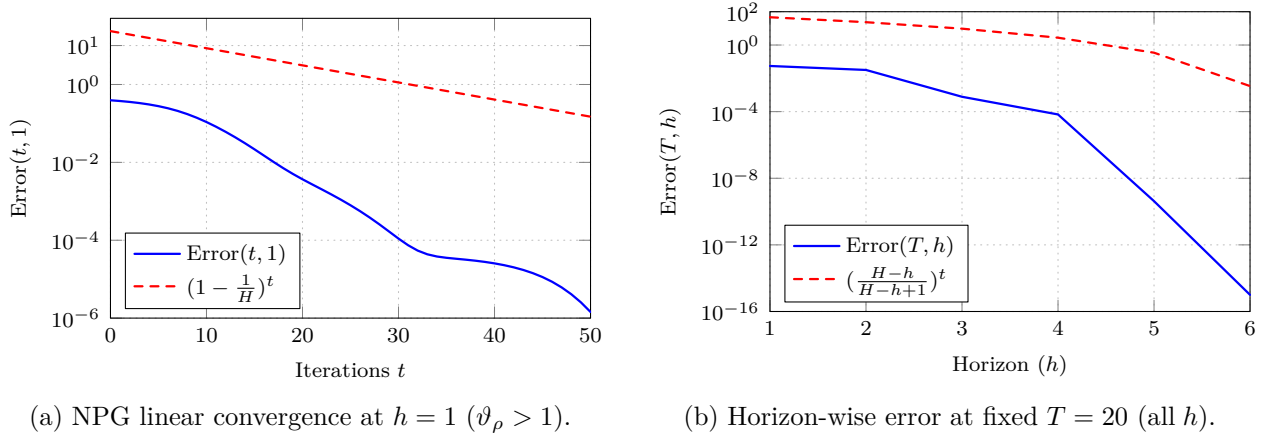
\begin{figure}[htbp]
    \centering

    \begin{subfigure}[t]{0.48\linewidth}
        \centering
        \begin{tikzpicture}
            \begin{semilogyaxis}[
                width=\linewidth,
                height=0.70\linewidth,
                xlabel={Iterations $t$},
                ylabel={Error$(t,1)$},
                xmin=0, xmax=50,
                xtick={0,10,20,30,40,50},
                ymin=1e-6, ymax=50,
                ytick={1e-6,1e-4,1e-2,1,10},
                grid=both,
                grid style={dotted},
    legend pos=south west,
    legend cell align={left},
    legend style={font=\scriptsize},
    tick label style={font=\scriptsize},
    label style={font=\scriptsize},
    line join=round,
    line cap=round
            ]
            \addplot+[blue, line width=0.85pt, no marks]
                table[col sep=comma, x=t, y=horizon1_distributional_gap]
                {increasing_h1_problem_dependent.csv};
            \addlegendentry{Error$(t,1)$}

            \addplot+[red, dashed, line width=0.85pt, no marks]
                table[col sep=comma, x=t, y=theorem_bound]
                {increasing_h1_problem_dependent.csv};
            \addlegendentry{$(1-\frac{1}{H})^t$}
            \end{semilogyaxis}
        \end{tikzpicture}
        \caption{NPG linear convergence at $h=1$ $(\vartheta_\rho >1)$.} \label{fig:app_inc_h1}
    \end{subfigure}
    \hfill
   \begin{subfigure}[t]{0.48\linewidth}
        \centering
        \begin{tikzpicture}
            \begin{semilogyaxis}[
                width=\linewidth,
                height=0.70\linewidth,
                xlabel={Horizon $(h)$},
                ylabel={Error$(T,h)$},
                xmin=1, xmax=6,
                xtick={1,2,3,4,5,6},
                ymin=1e-16, ymax=1e2,
                ytick={1e-16,1e-12,1e-8,1e-4,1,1e2},
                grid=both,
                grid style={dotted},
    legend pos=south west,
    legend cell align={left},
    legend style={font=\scriptsize},
    tick label style={font=\scriptsize},
    label style={font=\scriptsize},
    line join=round,
    line cap=round
            ]
            \addplot+[blue, line width=0.85pt, no marks]
                table[col sep=comma, x=h, y=actual_error_sum_T_h_plot]
                {increasing_horizonwise_problem_dependent_T20.csv};
            \addlegendentry{Error$(T,h)$}

            \addplot+[red, dashed, line width=0.85pt, no marks]
                table[col sep=comma, x=h, y=theory_bound_T_h]
                {increasing_horizonwise_problem_dependent_T20.csv};
            \addlegendentry{$(\frac{H-h}{H-h+1})^t$}
            \end{semilogyaxis}
        \end{tikzpicture}
        \caption{Horizon-wise error at fixed $T=20$ $(\text{all}\; h)$.}
        \label{fig:app_inc_all_h} 
    \end{subfigure}

    \caption{Increasing step size NPG on the randomly generated MDP under the problem-dependent $\vartheta_\rho > 1$ bound of Theorem~\ref{thm:geo_multistart}: (a) error vs.\ iteration at $h=1$; (b) error vs.\ all horizons at fixed iteration $T=20$. Both plots use linear $x$-axis and log-scale $y$-axis. The solid (blue) curve shows the empirical error, and the dashed (red) curve shows the corresponding geometric reference curve, which is undefined at $h=H$ where $\vartheta_\rho=1$.}
    \label{fig:increasing_h1_subfigures}
\end{figure}

We repeat the experiments of Section~\ref{subsec:increasing} on a randomly
generated finite-horizon MDP with the same dimensions
($|\mathcal{S}| = 15$, $|\mathcal{A}| = 5$, $H = 7$): each transition
distribution $P^h(\cdot \mid s, a)$ is drawn independently from the
symmetric Dirichlet distribution with unit concentration, and rewards are
i.i.d.\ uniform on $[0,1]$.  We use random seed $42$ for both the transitions
and rewards. Unlike the structured instance in Section~\ref{subsec:increasing}, the
uniform distribution need not be invariant under the transitions of this
MDP. We therefore evaluate the mismatch coefficient directly. Specifically,
we compute an optimal policy $\pi^\star$ by backward induction, construct
its multi-start visitation measures, and evaluate $\vartheta_\rho$ from these measures.

Figure~\ref{fig:app_inc_h1} plots $\mathrm{Error}(t,1)$ using problem-dependent bound of Theorem~\ref{thm:geo_multistart}, with
contraction factor $1 - 1/\vartheta_\rho$ and prefactor evaluated at the
measured $\vartheta_\rho$. As in the structured
experiment, we initialize $\pi_0$ uniformly.

Figure~\ref{fig:app_inc_all_h} repeats the
horizon-wise experiment: for each $h \in [H-1]$, we run NPG on the tail MDP with effective
horizon $H-h+1$ and evaluate $\vartheta_\rho$ directly on each tail. In both panels the empirical error remains below the predicted
bound.

\subsection{Linear Convergence with Robust Step Size Proofs} \label{app:linrobust}
\begin{proof}[Proof of Lemma~\ref{lem:vartheta_rho_ge_H}]
 Recall that $\bar d^{i,\pi^\star}_\rho$ is defined by
\[
\bar d^{i,\pi^\star}_\rho(s)=\frac{1}{H}\sum_{h_0=1}^i d^{h_0\to i,\pi^\star}_\rho(s).
\]
Since each $d^{h_0\to i,\pi^\star}_\rho(\cdot)$ is a probability distribution, we have
\[
\sum_{s\in\mathcal S}\bar d^{i,\pi^\star}_\rho(s)
=\frac{1}{H}\sum_{h_0=1}^i\sum_{s\in\mathcal S} d^{h_0\to i,\pi^\star}_\rho(s)
=\frac{i}{H}.
\]

Because $\vartheta_\rho<\infty$, the support condition holds:
\[
\bar d^{i,\pi^\star}_\rho(s)>0 \ \Longrightarrow\ \rho(s)>0.
\]
Thus the ratio $\bar d^{i,\pi^\star}_\rho(s)/\rho(s)$ is well-defined on $\{s:\rho(s)>0\}$, and
\[
\mathbb E_{S\sim \rho}\!\left[\frac{\bar d^{i,\pi^\star}_\rho(S)}{\rho(S)}\right]
=\sum_{s:\rho(s)>0}\rho(s)\frac{\bar d^{i,\pi^\star}_\rho(s)}{\rho(s)}
=\sum_{s\in\mathcal S}\bar d^{i,\pi^\star}_\rho(s)
=\frac{i}{H}.
\]
Therefore,
\[
\max_{s:\rho(s)>0}\frac{\bar d^{i,\pi^\star}_\rho(s)}{\rho(s)}
\ \ge\ \frac{i}{H}.
\]
Taking the maximum over $i\in[H]$ yields
\[
\max_{i\in[H]}\max_{s:\rho(s)>0}\frac{\bar d^{i,\pi^\star}_\rho(s)}{\rho(s)}
\ \ge\ 1,
\]
and multiplying by $H$ gives $\vartheta_\rho\ge H$ by \eqref{eq:vartheta_def_repeat}.

Finally, since the function $x\mapsto x/(x-1)$ is decreasing on $(1,\infty)$ and $\vartheta_\rho\ge H\ge 2$, we obtain
\[
\frac{\vartheta_\rho}{\vartheta_\rho-1}\ \le\ \frac{H}{H-1},
\]
which is \eqref{eq:ratio_vartheta_by_H}.
\end{proof}

\begin{lem}
\label{lem:horizon_only_growth_sufficient}
Fix $H\ge2$ and $\rho\in\Delta(\mathcal S)$. Let
$\{\pi_t\}_{t\ge0}$ be generated by the KL-based PMD
update~\eqref{eq:PMD}, with a positive step size sequence
$\{\eta_t\}_{t\ge0}$, from a full support initial policy. Let
$\pi^\star$ be an optimal nonstationary policy. Let
$\delta_0$ and $D_0^\star$ be as defined
in \eqref{eq:delta_with_J} and ~\eqref{eq:Dtstar_global} respectively. Suppose $\vartheta_\rho<\infty$. If the positive step sizes satisfy
\begin{equation}
\label{eq:eta_growth_H_lem}
\eta_{t+1}
\ge
\frac{H}{H-1}\eta_t,
\qquad t=0,1,2,\ldots.
\end{equation}
then the problem-dependent growth
condition~\eqref{eq:eta_growth_thm} is satisfied. Consequently, for
every $t\ge0$,
\begin{equation}
\label{eq:horizon_only_general_bound}
V^{\pi^\star,1}(\rho)-V^{\pi_t,1}(\rho)
\le
H\left(1-\frac{1}{\vartheta_\rho}\right)^t
\left(
\delta_0+
\frac{D_0^\star}{\eta_0(\vartheta_\rho-1)}
\right).
\end{equation}
\end{lem}

\begin{proof}
Since $H\ge2$ and $\vartheta_\rho<\infty$,
Lemma~\ref{lem:vartheta_rho_ge_H} gives
\[
\vartheta_\rho\ge H\ge2,
\]
and \eqref{eq:ratio_vartheta_by_H} gives
\[
\frac{\vartheta_\rho}{\vartheta_\rho-1}
\le
\frac{H}{H-1}.
\]
Since $\eta_t>0$, multiplying this inequality by $\eta_t$ yields, for
every $t\ge0$,
\[
\frac{\vartheta_\rho}{\vartheta_\rho-1}\eta_t
\le
\frac{H}{H-1}\eta_t.
\]
Combining this inequality with~\eqref{eq:eta_growth_H_lem}, we obtain
\[
\eta_{t+1}
\ge
\frac{H}{H-1}\eta_t
\ge
\frac{\vartheta_\rho}{\vartheta_\rho-1}\eta_t,
\qquad t\ge0.
\]
Therefore, the step size condition~\eqref{eq:eta_growth_thm} required
by Theorem~\ref{thm:geo_multistart} is satisfied.

Moreover, $\vartheta_\rho\ge H\ge2$ and
$\vartheta_\rho<\infty$ imply $\vartheta_\rho\in(1,\infty)$, while
the positivity of the step size sequence implies $\eta_0>0$.
Together with the remaining assumptions in the lemma statement, these
facts verify all the hypotheses of
Theorem~\ref{thm:geo_multistart}. Applying that theorem gives
\[
V^{\pi^\star,1}(\rho)-V^{\pi_t,1}(\rho)
\le
H\left(1-\frac{1}{\vartheta_\rho}\right)^t
\left(
\delta_0+
\frac{D_0^\star}{\eta_0(\vartheta_\rho-1)}
\right),
\qquad t\ge0,
\]
which proves~\eqref{eq:horizon_only_general_bound}.
\end{proof}

\end{document}